%% file: main.tex
\documentclass{article}
\usepackage{amssymb,amsmath,amsthm}
\usepackage{algorithm}
\usepackage{fullpage}
\usepackage{color}
\usepackage{bm}
\usepackage{linegoal}
\usepackage{calc}

\usepackage[noend]{algpseudocode}
\newcommand*{\LongState}[1]{\State
\parbox[t]{15pt+\linegoal}{#1\strut}}

\newtheorem{fact}{Fact}
\newtheorem{lemma}{Lemma}
\newtheorem{theorem}{Theorem}
\newtheorem{corollary}{Corollary}
\newtheorem{definition}{Definition}
\newtheorem{proposition}{Proposition}

\def\sign{\text{sign}}
\def\E{\mathbb{E}}
\def\P{\mathbb{P}}
\def\calZ{\mathcal{Z}}
\def\calF{\mathcal{F}}
\def\calO{\mathcal{O}}
\def\calY{\mathcal{Y}}
\def\calX{\mathcal{X}}

\def\calH{\mathcal{H}}
\def\calU{\mathcal{U}}

\def\calS{\mathcal{S}}
\def\err{\text{err}}

\def\R{\text{R}}
\def\DIS{\text{DIS}}

\def\endelta{\sigma(n,\delta)}
\def\enjdj{\sigma(n_j,\tilde{\delta}_j)}
\def\enjzdjz{\sigma(n_{j_0},\tilde{\delta}_{j_0})}

\def\errguar{\eta}

\def\argmin{\text{argmin}}

\title{Beyond Disagreement-based Agnostic Active Learning}

\author{
Chicheng Zhang\\
Computer Science and Engineering Department\\
University of California, San Diego \\
9500 Gilman Drive, La Jolla, CA 92093 \\
\texttt{chz038@eng.ucsd.edu} \\
\and
Kamalika Chaudhuri \\
Computer Science and Engineering Department\\
University of California, San Diego \\
9500 Gilman Drive, La Jolla, CA 92093 \\
\texttt{kamalika@cs.ucsd.edu} \\
}

\begin{document}
\maketitle

\begin{abstract}
We study agnostic active learning, where the goal is to learn a classifier in a pre-specified hypothesis class interactively with as few label queries as possible, while making no assumptions on the true function generating the labels. The main algorithms for this problem are {\em{disagreement-based active learning}}, which has a high label requirement, and {\em{margin-based active learning}}, which only applies to fairly restricted settings. A major challenge is to find an algorithm which achieves better label complexity, is consistent in an agnostic setting, and applies to general classification problems.

In this paper, we provide such an algorithm. Our solution is based on two novel contributions -- a reduction from consistent active learning to confidence-rated prediction with guaranteed error, and a novel confidence-rated predictor. 
\end{abstract}

\input{intro}

\input{prelims2}
\input{relwork}

\input{conclusions}
\paragraph{Acknowledgements.}
We thank NSF under IIS-1162581 for research support. We thank Sanjoy Dasgupta and Yoav Freund for helpful discussions. CZ would also like to thank Liwei Wang for introducing the problem of selective classification to him.
\bibliography{lpactive}
\bibliographystyle{alpha}
\input{appendix}

\input{derivation}

\end{document}

%% file: intro.tex
\section{Introduction}

In this paper, we study {\em{active learning}} of classifiers in an agnostic setting, where no assumptions are made on the true function that generates the labels. The learner has access to a large pool of unlabelled examples, and can interactively request labels for a small subset of these; the goal is to learn an accurate classifier in a pre-specified class with as few label queries as possible. Specifically, we are given a hypothesis class $\calH$ and a target $\epsilon$, and our aim is to find a binary classifier in $\calH$ whose error is at most $\epsilon$ more than that of the best classifier in $\calH$, while minimizing the number of requested labels.

There has been a large body of previous work on active learning; see the surveys by~\cite{D11, S10} for overviews. The main challenge in active learning is ensuring consistency in the agnostic setting while still maintaining low label complexity. In particular, a very natural approach to active learning is to view it as a generalization of binary search~\cite{FSST97, D05, N11}. While this strategy has been extended to several different noise models~\cite{K06, N11, NJC13}, it is generally inconsistent in the agnostic case~\cite{DH08}. 

The primary algorithm for agnostic active learning is called {\em{disagreement-based active learning}}. The main idea is as follows. A set $V_k$ of possible risk minimizers is maintained with time, and the label of an example $x$ is queried if there exist two hypotheses $h_1$ and $h_2$ in $V_k$ such that $h_1(x) \neq h_2(x)$. This algorithm is consistent in the agnostic setting~\cite{CAL94, BBL09, DHM07, H07, BDL09, H09, BHLZ10, K10}; however, due to the conservative label query policy, its label requirement is high. A line of work due to~\cite{BBZ07, BL13, ABL14} have provided algorithms that achieve better label complexity for linear classification on the uniform distribution over the unit sphere as well as log-concave distributions; however, their algorithms are limited to these specific cases, and it is unclear how to apply them more generally.

Thus, a major challenge in the agnostic active learning literature has been to find a general active learning strategy that applies to any hypothesis class and data distribution, is consistent in the agnostic case, and has a better label requirement than disagreement based active learning. This has been mentioned as an open problem by several works, such as~\cite{BBL09, D11, BL13}.

In this paper, we provide such an algorithm. Our solution is based on two key contributions, which may be of independent interest. The first is a general connection between {\em{confidence-rated predictors}} and active learning. A confidence-rated predictor is one that is allowed to abstain from prediction on occasion, and as a result, can guarantee a target prediction error. Given a confidence-rated predictor with guaranteed error, we show how to use it to construct an active label query algorithm consistent in the agnostic setting. Our second key contribution is a novel confidence-rated predictor with guaranteed error that applies to any general classification problem. We show that our predictor is {\em{optimal}} in the realizable case, in the sense that it has the lowest abstention rate out of all predictors that guarantee a certain error. Moreover, we show how to extend our predictor to the agnostic setting. 

Combining the label query algorithm with our novel confidence-rated predictor, we get a general active learning algorithm consistent in the agnostic setting. We provide a characterization of the label complexity of our algorithm, and show that this is better than disagreement-based active learning in general. Finally, we show that for linear classification with respect to the uniform distribution and log-concave distributions, our bounds reduce to those of~\cite{BBZ07, BL13}.

%% file: prelims2.tex
\section{Algorithm}
\label{sec:alg}

\subsection{The Setting} 
We study active learning for binary classification. Examples belong to an instance space $\calX$, and their labels lie in a label space $\calY = \{-1, 1\}$; labelled examples are drawn from an underlying data distribution $D$ on $\calX \times \calY$. We use $D_{\calX}$ to denote the marginal on $D$ on $\calX$, and $D_{Y|X}$ to denote the conditional distribution on $Y | X = x$ induced by $D$. Our algorithm has access to examples through two oracles -- an example oracle $\calU$ which returns an unlabelled example $x \in \calX$ drawn from $D_{\calX}$ and a labelling oracle $\calO$ which returns the label $y$ of an input $x \in \calX$ drawn from $D_{Y|X}$. 

Given a hypothesis class $\calH$ of VC dimension $d$, the error of any $h \in \calH$ with respect to a data distribution $\Pi$ over $\calX \times \calY$ is defined as $\err_{\Pi}(h) = \P_{(x, y) \sim \Pi}(h(x) \neq y)$. We define: $h^*(\Pi) = \argmin_{h \in \calH} \err_{\Pi}(h)$, $\nu^*(\Pi) = \err_{\Pi}(h^*(\Pi))$. For a set $S$, we abuse notation and use $S$ to also denote the uniform distribution over the elements of $S$. We define $\P_\Pi(\cdot):=\P_{(x,y) \sim \Pi}(\cdot)$, $\E_\Pi(\cdot):=\E_{(x,y) \sim \Pi}(\cdot)$.

Given access to examples from a data distribution $D$ through an example oracle $\calU$ and a labeling oracle $\calO$, we aim to provide a classifier $\hat{h} \in \calH$ such that with probability $\geq 1 - \delta$, $\err_D(\hat{h}) \leq \nu^*(D) + \epsilon$, for some target values of $\epsilon$ and $\delta$; this is achieved in an adaptive manner by making as few queries to the labelling oracle $\calO$ as possible. When $\nu^*(D) = 0$, we are said to be in the {\em{realizable case}}; in the more general {\em{agnostic}} case, we make no assumptions on the labels, and thus $\nu^*(D)$ can be positive.

Previous approaches to agnostic active learning have frequently used the notion of {\em{disagreements}}. The disagreement between two hypotheses $h_1$ and $h_2$ with respect to a data distribution $\Pi$ is the fraction of examples according to $\Pi$ to which $h_1$ and $h_2$ assign different labels; formally: $\rho_{\Pi}(h_1, h_2) = \P_{(x,y) \sim \Pi}(h_1(x) \neq h_2(x))$. Observe that a data distribution $\Pi$ induces a pseudo-metric $\rho_{\Pi}$ on the elements of $\calH$; this is called the disagreement metric. For any $r$ and any $h \in \calH$, define $B_{\Pi}(h, r)$ to be the disagreement ball of radius $r$ around $h$ with respect to the data distribution $\Pi$. Formally: $B_{\Pi}(h, r) = \{ h' \in \calH: \rho_{\Pi}(h, h') \leq r \}$. 

For notational simplicity, we assume that the hypothesis space is ``dense" with repsect to the data distribution $D$, in the sense that $\forall r > 0$, $\sup_{h \in B_D(h^*(D),r) } \rho_D(h,h^*(D)) = r$. Our analysis will still apply without the denseness assumption, but will be significantly more messy. Finally, given a set of hypotheses $V \subseteq \calH$, the {\em{disagreement region}} of $V$ is the set of all examples $x$ such that there exist two hypotheses $h_1, h_2 \in V$ for which $h_1(x) \neq h_2(x)$. 

This paper establishes a connection between active learning and confidence-rated predictors with guaranteed error. A confidence-rated predictor is a prediction algorithm that is occasionally allowed to abstain from classification. We will consider such predictors in the transductive setting. Given a set $V$ of candidate hypotheses, an error guarantee $\errguar$, and a set $U$ of unlabelled examples, a confidence-rated predictor $P$ either assigns a label or abstains from prediction on each unlabelled $x \in U$. The labels are assigned with the guarantee that the expected disagreement\footnote{where the expectation is with respect to the random choices made by $P$} between the label assigned by $P$ and any $h \in V$ is $\leq \errguar$. Specifically, 
\begin{equation}\label{eqn:errguar} 
{\text{for all\ }} h \in V, \quad \P_{x \sim U}(h(x) \neq P(x), P(x) \neq 0) \leq \errguar  
\end{equation}
This ensures that if some $h^* \in V$ is the true risk minimizer, then, the labels predicted by $P$ on $U$ do not differ very much from those predicted by $h^*$. The performance of a confidence-rated predictor which has a guarantee such as in Equation~\eqref{eqn:errguar} is measured by its {\em{coverage}}, or the probability of non-abstention $\P_{x \sim U}(P(x) \neq 0)$; higher coverage implies better performance.

\subsection{Main Algorithm}

Our active learning algorithm proceeds in epochs, where the goal of epoch $k$ is to achieve excess generalization error $\epsilon_k = \epsilon 2^{k_0 - k + 1}$, by querying a fresh batch of labels. The algorithm maintains a candidate set $V_k$ that is guaranteed to contain the true risk minimizer. 

The critical decision at each epoch is how to select a subset of unlabelled examples whose labels should be queried. We make this decision using a confidence-rated predictor $P$. At epoch $k$, we run $P$ with candidate hypothesis set $V = V_k$ and error guarantee $\errguar = \epsilon_k/64$. Whenever $P$ abstains, we query the label of the example. The number of labels $m_k$ queried is adjusted so that it is enough to achieve excess generalization error $\epsilon_{k+1}$.
 
An outline is described in Algorithm~\ref{alg:labelquery}; we next discuss each individual component in detail.
\begin{algorithm}[H]
\caption{Active Learning Algorithm: Outline}
\label{alg:labelquery}
\begin{algorithmic}[1]
\State {\bf{Inputs:}} Example oracle $\calU$, Labelling oracle $\calO$, hypothesis class $\calH$ of VC dimension $d$, confidence-rated predictor $P$, target excess error $\epsilon$ and target confidence $\delta$.
\State Set $k_0 = \lceil \log{1/\epsilon} \rceil$. Initialize candidate set $V_1 = \calH$. 
\For{$k = 1, 2, .. k_0$}
\State Set $\epsilon_k = \epsilon 2^{k_0 - k +1}$, $\delta_k = \frac{\delta}{2(k_0-k+1)^2}$. 
\LongState{Call $\calU$ to generate a fresh unlabelled sample $U_k = \{z_{k,1}, ..., z_{k,n_k}\}$ of size $n_k=192(\frac{256}{\epsilon_k})^2(d\ln\frac{256}{\epsilon_k} + \ln\frac{288}{\delta_k})$.}
\LongState{Run confidence-rated predictor $P$ with inputs $V=V_k$, $U=U_k$ and error guarantee $\eta = \epsilon_k/64$ to get abstention probabilities $\gamma_{k,1}, \ldots, \gamma_{k,n_k}$ on the examples in $U_k$. These probabilities induce a distribution $\Gamma_k$ on $U_k$. Let $\phi_k = \P_{x \sim U_k}(P(x) = 0) = \frac{1}{n_k} \sum_{i=1}^{n_k} \gamma_{k,i}$.} 

\If{in the Realizable Case}
\LongState{Let $m_k = \frac{768\phi_k}{\epsilon_k}(d\ln\frac{768\phi_k}{\epsilon_k} + \ln\frac{48}{\delta_k})$. Draw $m_k$ i.i.d examples from $\Gamma_k$ and query $\calO$ for labels of these examples to get a labelled data set $S_k$. Update $V_{k+1}$ using $S_k$: $V_{k+1} := \{h \in V_k: h(x) = y, \text{ for all } (x,y) \in S_k \}$.}
\Else 
\LongState{In the non-realizable case, use Algorithm~\ref{alg:adaptive} with inputs hypothesis set $V_k$, distribution $\Gamma_k$, target excess error $\frac{\epsilon_k}{8\phi_k}$, target confidence $\frac{\delta_k}{2}$, and the labeling oracle $\calO$ to get a new hypothesis set $V_{k+1}$.}
\EndIf
\EndFor
\State \Return an arbitrary $\hat{h} \in V_{k_0+1}$.
\end{algorithmic}
\end{algorithm}

\paragraph{Candidate Sets.}
At epoch $k$, we maintain a set $V_k$ of candidate hypotheses guaranteed to contain the true risk minimizer $h^*(D)$ (w.h.p).  In the realizable case, we use a version space as our candidate set. The version space with respect to a set $S$ of labelled examples is the set of all $h \in \calH$ such that $h(x_i) = y_i$ for all $(x_i, y_i) \in S$.

\begin{lemma}
\label{lem:vsrealizable}
Suppose we run Algorithm~\ref{alg:labelquery} in the realizable case with inputs example oracle $\calU$, labelling oracle $\calO$, hypothesis class $\calH$, confidence-rated predictor $P$, target excess error $\epsilon$ and target confidence $\delta$. Then, with probability $1$, $h^*(D) \in V_k, \text{ for all } k = 1, 2, \ldots, k_0+1$.
\end{lemma}

In the non-realizable case, the version space is usually empty; we use instead a $(1 - \alpha)$-confidence set for the true risk minimizer. Given a set $S$ of $n$ labelled examples, let $C(S) \subseteq \calH$ be a function of $S$; $C(S)$ is said to be a $(1 - \alpha)$-confidence set for the true risk minimizer if for all data distributions $\Delta$ over $\calX \times \calY$,
\[ \P_{S \sim \Delta^n}[h^*(\Delta) \in C(S)] \geq 1 - \alpha,\]
Recall that $h^*(\Delta) = \argmin_{h \in \calH} \err_{\Delta}(h)$. In the non-realizable case, our candidate sets are $(1 - \alpha)$-confidence sets for $h^*(D)$, for $\alpha = \delta$. The precise setting of $V_k$ is explained in Algorithm~\ref{alg:adaptive}.

\begin{lemma}
\label{lem:vsnonrealizable}
Suppose we run Algorithm~\ref{alg:labelquery} in the non-realizable case with inputs example oracle $\calU$, labelling oracle $\calO$, hypothesis class $\calH$, confidence-rated predictor $P$, target excess error $\epsilon$ and target confidence $\delta$. Then with probability $1 - \delta$, $h^*(D) \in V_k, \text{ for all } k = 1, 2, \ldots, k_0+1$.
\end{lemma}

\paragraph{Label Query.} We next discuss our label query procedure -- which examples should we query labels for, and how many labels should we query at each epoch?

\paragraph{Which Labels to Query?} Our goal is to query the labels of the most informative examples. To choose these examples while still maintaining consistency, we use a confidence-rated predictor $P$ with guaranteed error. The inputs to the predictor are our candidate hypothesis set $V_k$ which contains (w.h.p) the true risk minimizer, a fresh set $U_k$ of unlabelled examples, and an error guarantee $\errguar = \epsilon_k/64$. For notation simplicity, assume the elements in $U_k$ are distinct. The output is a sequence of abstention probabilities $\{ \gamma_{k,1}, \gamma_{k,2}, \ldots, \gamma_{k,n_k}\}$, for each example in $U_k$. It induces a distribution $\Gamma_k$ over $U_k$, from which we independently draw examples for label queries.

\paragraph{How Many Labels to Query?} The goal of epoch $k$ is to achieve excess generalization error $\epsilon_k$. To achieve this, passive learning requires $\tilde{O}(d/\epsilon_k)$ labelled examples\footnote{$\tilde{O}(\cdot)$ hides logarithmic factors} in the realizable case, and $\tilde{O}(d (\nu^*(D) + \epsilon_k)/\epsilon_k^2)$ examples in the agnostic case. A key observation in this paper is that in order to achieve excess generalization error $\epsilon_{k}$ on $D$, it suffices to achieve a much larger excess generalization error $O(\epsilon_{k}/\phi_k)$ on the data distribution induced by $\Gamma_k$ and $D_{Y|X}$, where $\phi_k$ is the fraction of examples on which the confidence-rated predictor abstains. 

In the realizable case, we achieve this by sampling $m_k = \frac{768\phi_k}{\epsilon_k}(d\ln\frac{768\phi_k}{\epsilon_k} + \ln\frac{48}{\delta_k})$ i.i.d examples from $\Gamma_k$, and querying their labels to get a labelled dataset $S_k$. Observe that as $\phi_k$ is the abstention probability of $P$ with guaranteed error $\leq \epsilon_k/64$, it is generally smaller than the measure of the disagreement region of the version space; this key fact results in improved label complexity over disagreement-based active learning. This sampling procedure has the following property:

\begin{lemma}
\label{lem:errdecrrealizable}
Suppose we run Algorithm~\ref{alg:labelquery} in the realizable case with inputs example oracle $\calU$, labelling oracle $\calO$, hypothesis class $\calH$, confidence-rated predictor $P$, target excess error $\epsilon$ and target confidence $\delta$. Then with probability $1 - \delta$, for all $k = 1, 2, \ldots, k_0+1$, and for all $h \in V_k$, $\err_D(h) \leq \epsilon_k$. In particular, the $\hat{h}$ returned at the end of the algorithm satisfies $\err_D(\hat{h}) \leq \epsilon$.
\end{lemma}

The agnostic case has an added complication -- in practice, the value of $\nu^*$ is not known ahead of time. Inspired by~\cite{K10}, we use a {\em{doubling procedure}}(stated in Algorithm~\ref{alg:adaptive}) which adaptively finds the number $m_k$ of labelled examples to be queried and queries them. The following two lemmas illustrate its properties -- that it is consistent, and that it does not use too many label queries.
 
\begin{lemma}
\label{lem:ratiotype}
Suppose we run Algorithm~\ref{alg:adaptive} with inputs hypothesis set $V$, example distribution $\Delta$, labelling oracle $\calO$, target excess error $\tilde{\epsilon}$ and target confidence $\tilde{\delta}$. Let $\tilde{\Delta}$ be the joint distribution on $\calX \times \calY$ induced by $\Delta$ and $D_{Y|X}$. Then there exists an event $\tilde{E}$, $\P(\tilde{E}) \geq 1 - \tilde{\delta}$, such that on $\tilde{E}$, (1) Algorithm~\ref{alg:adaptive} halts and (2) the set $V_{j_0}$ has the following properties: 

(2.1) If for $h \in \calH$, $\err_{\tilde{\Delta}}(h) - \err_{\tilde{\Delta}}(h^*(\tilde{\Delta})) \leq \tilde{\epsilon}/2$, then $h \in V_{j_0}$. 

(2.2) On the other hand, if $h \in V_{j_0}$, then $\err_{\tilde{\Delta}}(h) - \err_{\tilde{\Delta}}(h^*(\tilde{\Delta})) \leq \tilde{\epsilon}$.
\end{lemma}

When event $\tilde{E}$ happens, we say Algorithm~\ref{alg:adaptive} succeeds.
\begin{lemma}
\label{lem:adaptivetonu}
Suppose we run Algorithm~\ref{alg:adaptive} with inputs hypothesis set $V$, example distribution $\Delta$, labelling oracle $\calO$, target excess error $\tilde{\epsilon}$ and target confidence $\tilde{\delta}$. There exists some absolute constant $c_1 > 0$, such that on the event that Algorithm~\ref{alg:adaptive} succeeds, $n_{j_0} \leq c_1((d\ln\frac{1}{\tilde{\epsilon}} + \ln\frac{1}{\tilde{\delta}} )\frac{\nu^*(\tilde{\Delta}) + \tilde{\epsilon}}{\tilde{\epsilon}^2})$. Thus the total number of labels queried is $\sum_{j=1}^{j_0} n_j \leq 2 n_{j_0} \leq 2c_1((d\ln\frac{1}{\tilde{\epsilon}} + \ln\frac{1}{\tilde{\delta}} )\frac{\nu^*(\tilde{\Delta}) + \tilde{\epsilon}}{\tilde{\epsilon}^2})$.
\end{lemma}
A naive approach (see Algorithm~\ref{alg:nonadaptive} in the Appendix) which uses an additive VC bound gives a sample complexity of $O((d\ln(1/\tilde{\epsilon})+\ln(1/\tilde{\delta})) \tilde{\epsilon}^{-2})$; Algorithm~\ref{alg:adaptive} gives a better sample complexity.

The following lemma is a consequence of our label query procedure in the non-realizable case.

\begin{lemma}
\label{lem:errdecrnonrealizable}
Suppose we run Algorithm~\ref{alg:labelquery} in the non-realizable case with inputs example oracle $\calU$, labelling oracle $\calO$, hypothesis class $\calH$, confidence-rated predictor $P$, target excess error $\epsilon$ and target confidence $\delta$. Then with probability $1 - \delta$, for all $k = 1, 2, \ldots, k_0+1$, and for all $h \in V_k$, $\err_D(h) \leq \err_D(h^*(D)) + \epsilon_k$. In particular, the $\hat{h}$ returned at the end of the algorithm satisfies $\err_D(\hat{h}) \leq \err_D(h^*(D)) + \epsilon$.
\end{lemma}

\begin{algorithm}
\caption{An Adaptive Algorithm for Label Query Given Target Excess Error}
\label{alg:adaptive}
\begin{algorithmic}[1]
\State {\bf{Inputs:}} Hypothesis set $V$ of VC dimension $d$, Example distribution $\Delta$, Labeling oracle $\calO$, target excess error $\tilde{\epsilon}$, target confidence $\tilde{\delta}$.
\For{$j = 1, 2, \ldots$} 
\LongState{Draw $n_j = 2^j$ i.i.d examples from $\Delta$; query their labels from $\calO$ to get a labelled dataset $S_j$. Denote $\tilde{\delta}_j := \tilde{\delta}/(j(j+1))$.}
\State Train an ERM classifier $\hat{h}_j \in V$ over $S_j$.
\LongState{Define the set $V_j$ as follows: 
\[ V_j = \Big{\{}h \in V: \err_{S_j}(h) \leq \err_{S_j}(\hat{h}_j) +  \frac{\tilde{\epsilon}}{2} + \enjdj + \sqrt{\enjdj \rho_{S_j}(h, \hat{h}_j)} \Big{\}} \]
Where $\sigma(n,\delta) := \frac{8}{n} (2d\ln \frac{2en}{d} + \ln \frac{24}{\delta})$.}
\If {$\label{eqn:stopcriterion}\sup_{h \in V_j} ( \enjdj + \sqrt{\enjdj \rho_{S_j}(h, \hat{h}_j)} ) \leq \frac{\tilde{\epsilon}}{6}$}
    \State $j_0 = j$, \textbf{break}
\EndIf
\EndFor
\State \Return $V_{j_0}$.
\end{algorithmic}
\end{algorithm}

\vspace{0cm}
\subsection{Confidence-Rated Predictor}
\vspace{0cm}
Our active learning algorithm uses a confidence-rated predictor with guaranteed error to make its label query decisions. In this section, we provide a novel confidence-rated predictor with guaranteed error. This predictor has optimal coverage in the realizable case, and may be of independent interest. The predictor $P$ receives as input a set $V \subseteq \calH$ of hypotheses (which is likely to contain the true risk minimizer), an error guarantee $\eta$, and a set of $U$ of unlabelled examples. We consider a {\em{soft prediction algorithm}}; so, for each example in $U$, the predictor $P$ outputs three probabilities that add up to $1$ -- the probability of predicting $1$, $-1$ and $0$. This output is subject to the constraint that the expected disagreement\footnote{where the expectation is taken over the random choices made by $P$} between the $\pm 1$ labels assigned by $P$ and those assigned by any $h \in V$ is at most $\errguar$, and the goal is to maximize the coverage, or the expected fraction of non-abstentions.

Our key insight is that this problem can be written as a linear program, which is described in Algorithm~\ref{alg:optcrp}. There are three variables, $\xi_i$, $\zeta_i$ and $\gamma_i$, for each unlabelled $z_i \in U$; there are the probabilities with which we predict $1$, $-1$ and $0$ on $z_i$ respectively. Constraint~\eqref{eqn:errconstraint} ensures that the expected disagreement between the label predicted and any $h \in V$ is no more than $\errguar$, while the LP objective maximizes the coverage under these constraints. Observe that the LP is always feasible. Although the LP has infinitely many constraints, the number of constraints in Equation~\eqref{eqn:errconstraint} distinguishable by $U_k$ is at most $(em/d)^d$, where $d$ is the VC dimension of the hypothesis class $\calH$. 

\begin{algorithm}
\caption{Confidence-rated Predictor}
\label{alg:optcrp}
\begin{algorithmic}[1]
\State {\bf{Inputs:}} hypothesis set $V$, unlabelled data $U = \{ z_1, \ldots, z_m \}$, error bound $\errguar$.
\State Solve the linear program:
\begin{align}
\min\; & \sum_{i=1}^{m} \gamma_i \nonumber \\
{\text{subject to:}} \quad \forall i, \; \; & \xi_i + \zeta_i + \gamma_i = 1 \nonumber \\
\forall h \in V, \; \; & \sum_{i: h(z_i) = 1} \zeta_i + \sum_{i: h(z_i) = -1} \xi_i \leq \errguar m \label{eqn:errconstraint} \\
\forall i, \; \; & \xi_i, \zeta_i, \gamma_i \geq 0 \nonumber 
\end{align}
\State For each $z_i \in U$, output probabilities for predicting $1$, $-1$ and $0$: $\xi_i$, $\zeta_i$, and $\gamma_i$.  
\end{algorithmic}
\end{algorithm}

The performance of a confidence-rated predictor is measured by its error and coverage. The error of a confidence-rated predictor is the probability with which it predicts the wrong label on an example, while the coverage is its probability of non-abstention. We can show the following guarantee on the performance of the predictor in Algorithm~\ref{alg:optcrp}.

\begin{theorem} 
In the realizable case, if the hypothesis set $V$ is the version space with respect to a training set, then $\P_{x \sim U}(P(x) \neq h^*(x), P(x) \neq 0) \leq \errguar$. In the non-realizable case, if the hypothesis set $V$ is an $(1- \alpha)$-confidence set for the true risk minimizer $h^*$, then, w.p $\geq 1 - \alpha$, $\P_{x \sim U}(P(x) \neq y, P(x) \neq 0) \leq \P_{x \sim U}(h^*(x) \neq y)  + \errguar$.
\label{thm:optcrperr}
\end{theorem}

In the realizable case, we can also show that our confidence rated predictor has optimal coverage. Observe that we cannot directly show optimality in the non-realizable case, as the performance depends on the exact choice of the $(1 - \alpha)$-confidence set.

\begin{theorem}
In the realizable case, suppose that the hypothesis set $V$ is the version space with respect to a training set. If $P'$ is any confidence rated predictor with error guarantee $\errguar$, and if $P$ is the predictor in Algorithm~\ref{alg:optcrp}, then, the coverage of $P$ is at least much as the coverage of $P'$.
\label{thm:optcrpcov}
\end{theorem}

\section{Performance Guarantees}
\label{sec:performgrt}

An essential property of any active learning algorithm is consistency -- that it converges to the true risk minimizer given enough labelled examples. We observe that our algorithm is consistent provided we use {\em{any}} confidence-rated predictor $P$ with guaranteed error as a subroutine. The consistency of our algorithm is a consequence of Lemmas~\ref{lem:errdecrrealizable} and~\ref{lem:errdecrnonrealizable} and is shown in Theorem~\ref{thm:consistency}. 

\begin{theorem}[Consistency]\label{thm:consistency}
Suppose we run Algorithm~\ref{alg:labelquery} with inputs example oracle $\calU$, labelling oracle $\calO$, hypothesis class $\calH$, confidence-rated predictor $P$, target excess error $\epsilon$ and target confidence $\delta$. Then with probability $1-\delta$, the classifier $\hat{h}$ returned by Algorithm~\ref{alg:labelquery} satisfies $\err_D(\hat{h}) - \err_D(h^*(D)) \leq \epsilon$.
\end{theorem}

We now establish a label complexity bound for our algorithm; however, this label complexity bound applies only if we use the predictor described in Algorithm~\ref{alg:optcrp} as a subroutine. 

For any hypothesis set $V$, data distribution $D$, and $\errguar$, define $\bm{\Phi}_D(V, \errguar)$ to be the minimum abstention probability of a confidence-rated predictor which guarantees that the disagreement between its predicted labels and any $h \in V$ under $D_{\calX}$ is at most $\errguar$. 

Formally, $\bm{\Phi}_D(V,\errguar) = \min\{ \E_D \gamma(x): \E_D [I(h(x)=+1)\zeta(x) + I(h(x)=-1)\xi(x)] \leq \errguar \text{ for all } h \in V, \gamma(x) + \xi(x) + \zeta(x) \equiv 1, \gamma(x), \xi(x), \zeta(x) \geq 0 \}$. Define $\phi(r, \errguar) := \bm{\Phi}_D(B_D(h^*, r), \errguar)$. The label complexity of our active learning algorithm can be stated as follows.

\begin{theorem}[Label Complexity]\label{thm:labelcomplexity}
Suppose we run Algorithm~\ref{alg:labelquery} with inputs example oracle $\calU$, labelling oracle $\calO$, hypothesis class $\calH$, confidence-rated predictor $P$ of Algorithm~\ref{alg:optcrp}, target excess error $\epsilon$ and target confidence $\delta$. Then there exist constants $c_3, c_4 > 0$ such that with probability $1-\delta$:\\
(1) In the realizable case, the total number of labels queried by Algorithm~\ref{alg:labelquery} is at most:
\[ c_3 \sum_{k=1}^{\lceil \log\frac{1}{\epsilon} \rceil} (d\ln\frac{\phi(\epsilon_k,\epsilon_k/256)}{\epsilon_k} + \ln( \frac{\lceil \log(1/\epsilon) \rceil - k + 1}{\delta})) \frac{\phi(\epsilon_k,\epsilon_k/256)}{\epsilon_k} \]
(2) In the agnostic case, the total number of labels queried by Algorithm~\ref{alg:labelquery} is at most:
\[ c_4 \sum_{k=1}^{\lceil \log\frac{1}{\epsilon} \rceil} ( d\ln \frac{\phi(2\nu^*(D) + \epsilon_k,\epsilon_k/256)}{\epsilon_k}  + \ln(  \frac{ \lceil\log(1/\epsilon) \rceil - k + 1}{\delta}) ) \frac{\phi(2\nu^*(D) + \epsilon_k,\epsilon_k/256)}{\epsilon_k}  (1 + \frac{\nu^*(D)}{\epsilon_k})  \]
\end{theorem}

\paragraph{Comparison.} The label complexity of disagreement-based active learning is characterized in terms of the {\em{disagreement coefficient}}. Given a radius $r$, the disagreement coefficent $\theta(r)$ is defined as: 
\[ \theta(r) = \sup_{r' \geq r} \frac{\P(\DIS(B_D(h^*, r')))}{r'}, \]
where for any $V \subseteq \calH$, $\DIS(V)$ is the disagreement region of $V$. As $\P(\DIS(B_D(h^*, r))) = \phi(r, 0)$~\cite{EYW10}, in our notation, $\theta(r) = \sup_{r' \geq r} \frac{\phi(r', 0)}{r'}$. 

In the realizable case, the label complexity of disagreement-based active learning is $\tilde{O}(\theta(\epsilon) \cdot \ln(1/\epsilon) \cdot (d \ln \theta(\epsilon) + \ln \ln (1/\epsilon)))$~\cite{H13}\footnote{Here the $\tilde{O}()$ notation hides factors logarithmic in $1/\delta$}. Our label complexity bound may be simplified to: 
\[ \tilde{O}\left( \ln\frac{1}{\epsilon}\cdot \sup_{k \leq \lceil \log(1/\epsilon)\rceil } \frac{\phi(\epsilon_k, \epsilon_k/256)}{\epsilon_k} \cdot \left(d \ln \left(\sup_{k \leq \lceil \log(1/\epsilon) \rceil} \frac{\phi(\epsilon_k, \epsilon_k/256)}{\epsilon_k}\right) + \ln\ln\frac{1}{\epsilon}\right) \right), \]
which is essentially the bound of~\cite{H13} with $\theta(\epsilon)$ replaced by $\sup_{k \leq \lceil \log(1/\epsilon) \rceil} \frac{\phi(\epsilon_k, \epsilon_k/256)}{\epsilon_k}$. As enforcing a lower error guarantee requires more abstention, $\phi(r, \eta)$ is a decreasing function of $\eta$; as a result, 
\[ \sup_{k \leq \lceil \log(1/\epsilon)\rceil } \frac{\phi(\epsilon_k, \epsilon_k/256)}{\epsilon_k} \leq \theta(\epsilon), \]
and our label complexity is better. 

In the agnostic case, \cite{DHM07} provides a label complexity bound of $\tilde{O}(\theta(2 \nu^*(D) + \epsilon) \cdot (d \frac{\nu^*(D)^2}{\epsilon^2} \ln(1/\epsilon) + d \ln^2(1/\epsilon)))$ for disagreement-based active-learning. In contrast, by Proposition~\ref{prop:agnosticlc} our label complexity is at most: 
\[ \tilde{O}\left( \sup_{k \leq \lceil \log(1/\epsilon) \rceil} \frac{\phi(2 \nu^*(D) + \epsilon_k, \epsilon_k/256)}{2 \nu^*(D) + \epsilon_k}   \cdot \left( d \frac{\nu^*(D)^2}{\epsilon^2} \ln(1/\epsilon) + d \ln^2(1/\epsilon)\right) \right) \]
Again, this is essentially the bound of~\cite{DHM07} with $\theta(2 \nu^*(D) + \epsilon)$ replaced by the smaller quantity
\[ \sup_{k \leq \lceil \log(1/\epsilon)\rceil} \frac{\phi(2 \nu^*(D) + \epsilon_k, \epsilon_k/256)}{2 \nu^*(D) + \epsilon_k}, \]

\cite{H13} has provided a more refined analysis of disagreement-based active learning that gives a label complexity  of $\tilde{O}( \theta(\nu^*(D) + \epsilon) (\frac{\nu^*(D)^2}{\epsilon^2} + \ln\frac{1}{\epsilon}) (d \ln\theta(\nu^*(D) + \epsilon) + \ln\ln\frac{1}{\epsilon}) )$; observe that their dependence is still on $\theta(\nu^*(D)+\epsilon)$. We leave a more refined label complexity analysis of our algorithm for future work. 

\subsection{Tsybakov Noise Conditions}

An important sub-case of learning from noisy data is learning under the Tsybakov noise conditions~\cite{T04}.  
\begin{definition}\label{def:tnc}(Tsybakov Noise Condition)
Let $\kappa \geq 1$. A labelled data distribution $D$ over $\calX \times \calY$ satisfies $(C_0,\kappa)$-Tsybakov Noise Condition with respect to a hypothesis class $\calH$ for some constant $C_0 > 0$, if for all $h \in \calH$, $\rho_D(h,h^*(D)) \leq C_0 (\err_D(h) - \err_D(h^*(D)))^{\frac{1}{\kappa}}$.
\end{definition}
The following theorem shows the performance guarantees achieved by Algorithm~\ref{alg:labelquery} under the Tsybakov noise conditions.

\begin{theorem}\label{thm:labelcomplexitytnc}
Suppose $(C_0,\kappa)$-Tsybakov Noise Condition holds for $D$ with respect to $\calH$. Then Algorithm~\ref{alg:labelquery} with inputs example oracle $\calU$, labelling oracle $\calO$, hypothesis class $\calH$, confidence-rated predictor $P$ of Algorithm~\ref{alg:optcrp}, target excess error $\epsilon$ and target confidence $\delta$ satisfies the following properties. There exists a constant $c_5 > 0$ such that with probability $1-\delta$, the total number of labels queried by Algorithm~\ref{alg:labelquery} is at most:
\[ c_5\sum_{k=1}^{\lceil \log\frac{1}{\epsilon} \rceil} (d\ln (\phi(C_0\epsilon_k^{\frac{1}{\kappa}},\epsilon_k/256) \epsilon_k^{\frac{1}{\kappa}-2}) + \ln(\frac{\lceil \log\frac{1}{\epsilon} \rceil - k + 1}{\delta})) \phi(C_0\epsilon_k^{\frac{1}{\kappa}},\epsilon_k/256) \epsilon_k^{\frac{1}{\kappa}-2} \]
\end{theorem}

\paragraph{Comparison.} \cite{H13} provides a label complexity bound of $\tilde{O}(\theta(C_0 \epsilon^{\frac{1}{\kappa}}) \epsilon^{\frac{2}{\kappa} - 2} \ln\frac{1}{\epsilon} (d\ln\theta(C_0 \epsilon^{\frac{1}{\kappa}}) + \ln\ln\frac{1}{\epsilon} ))$ for disagreement-based active learning. For $\kappa > 1$, by Proposition~\ref{prop:tnclc}, our label complexity is at most:

\[ \tilde{O}\left(\sup_{k \leq \lceil \log(1/\epsilon) \rceil} \frac{\phi(C_0 \epsilon_k^{1/\kappa}, \epsilon_k/256)}{\epsilon_k^{1/\kappa}} \cdot \epsilon_k^{2/\kappa - 2} \cdot d \ln(1/\epsilon)\right), \]

For $\kappa = 1$, our label complexity is at most 
\[ \tilde{O}\left(\ln\frac{1}{\epsilon}\cdot \sup_{k \leq \lceil \log(1/\epsilon) \rceil} \frac{\phi(C_0\epsilon_k, \epsilon_k/256)}{\epsilon_k} \cdot \left(d \ln(\sup_{k \leq \lceil \log(1/\epsilon) \rceil} \frac{\phi(C_0\epsilon_k, \epsilon_k/256)}{\epsilon_k}) + \ln\ln\frac{1}{\epsilon}\right) \right). \]
 In both cases, our bounds are better, as $\sup_{k \leq \lceil \log(1/\epsilon) \rceil} \cdot \frac{\phi(C_0 \epsilon_k^{1/\kappa}, \epsilon_k/256)}{C_0\epsilon_k^{1/\kappa}} \leq \theta(C_0 \epsilon^{1/\kappa})$.  In further work,~\cite{HY12} provides a refined analysis with a bound of $\tilde{O}(\theta(C_0 \epsilon^{\frac{1}{\kappa}}) \epsilon^{\frac{2}{\kappa} - 2}\ d\ln\theta(C_0 \epsilon^{\frac{1}{\kappa}})) $; however, this work is not directly comparable to ours, as they need prior knowledge of $C_0$ and $\kappa$.

\subsection{Case Study: Linear Classification under the Log-concave Distribution}
\vspace{0cm}

We now consider learning linear classifiers with respect to log-concave data distribution on $\R^d$. In this case, for any $r$, the disagreement coefficient $\theta(r) \leq O(\sqrt{d} \ln(1/r))$~\cite{BL13}; however, for any $\eta > 0$, $\frac{\phi(r, \eta)}{r} \leq O(\ln(r/\eta))$ (see Lemma~\ref{lem:logconcavephi} in the Appendix), which is much smaller so long as $\eta/r$ is not too small. This leads to the following label complexity bounds. 

\begin{corollary} \label{cor:logconcave}
Suppose $D_{\calX}$ is isotropic and log-concave on $\R^d$, and $\calH$ is the set of homogeneous linear classifiers on $\R^d$. Then Algorithm~\ref{alg:labelquery} with inputs example oracle $\calU$, labelling oracle $\calO$, hypothesis class $\calH$, confidence-rated predictor $P$ of Algorithm~\ref{alg:optcrp}, target excess error $\epsilon$ and target confidence $\delta$ satisfies the following properties.
With probability $1-\delta$:\\
(1) In the realizable case, there exists some absolute constant $c_8 > 0$ such that the total number of labels queried is at most $c_8 \ln\frac{1}{\epsilon} (d + \ln\ln\frac{1}{\epsilon} + \ln\frac{1}{\delta})$.\\
(2) In the agnostic case, there exists some absolute constant $c_9 > 0$ such that the total number of labels queried is at most $c_9 (\frac{{\nu^*(D)}^2}{\epsilon^2} + \ln\frac{1}{\epsilon}) \ln\frac{\epsilon + \nu^*(D)}{\epsilon}  (d\ln\frac{\epsilon + \nu^*(D)}{\epsilon} + \ln{\frac{1}{\delta}}) + \ln\frac{1}{\epsilon} \ln\frac{\epsilon + \nu^*(D)}{\epsilon} \ln\ln\frac{1}{\epsilon}$.\\
(3) If $(C_0, \kappa)$-Tsybakov Noise condition holds for $D$ with respect to $\calH$, then there exists some constant $c_{10} > 0$ (that depends on $C_0,\kappa$) such that the total number of labels queried is at most $c_{10} \epsilon^{\frac{2}{\kappa}-2} \ln\frac{1}{\epsilon} (d\ln\frac{1}{\epsilon} + \ln{\frac{1}{\delta}}) $.
\end{corollary}

In the realizable case, our bound matches~\cite{BL13}. For disagreement-based algorithms, the bound is $\tilde{O}(d^{\frac{3}{2}}\ln^2\frac{1}{\epsilon}(\ln d + \ln\ln\frac{1}{\epsilon}))$, which is worse by a factor of $O(\sqrt{d} \ln(1/\epsilon))$. \cite{BL13} does not address the fully agnostic case directly; however, if $\nu^*(D)$ is known a-priori, then their algorithm can achieve roughly the same label complexity as ours. 

For the Tsybakov Noise Condition with $\kappa > 1$,~\cite{BBZ07,BL13} provides a label complexity bound for $\tilde{O}(\epsilon^{\frac{2}{\kappa}-2}\ln^2\frac{1}{\epsilon}(d + \ln\ln\frac{1}{\epsilon}))$ with an algorithm that has a-priori knowledge of $C_0$ and $\kappa$. We get a slightly better bound. On the other hand, a disagreement based algorithm~\cite{H13} gives a label complexity of $\tilde{O}(d^{\frac{3}{2}} \ln^2\frac{1}{\epsilon} \epsilon^{\frac{2}{\kappa}-2} (\ln d + \ln\ln\frac{1}{\epsilon}))$. Again our bound is better by factor of $\Omega(\sqrt{d})$ over disagreement-based algorithms. For $\kappa = 1$, we can tighten our label complexity to get a $\tilde{O}(\ln\frac{1}{\epsilon} (d + \ln\ln\frac{1}{\epsilon} + \ln\frac{1}{\delta}))$ bound, which again matches~\cite{BL13}, and is better than the ones provided by disagreement-based algorithm -- $\tilde{O}(d^{\frac{3}{2}}\ln^2\frac{1}{\epsilon}(\ln d + \ln\ln\frac{1}{\epsilon} ))$~\cite{H13}.

%% file: relwork.tex
\vspace{0cm}
\section{Related Work}
\vspace{0cm}
Active learning has seen a lot of progress over the past two decades, motivated by vast amounts of unlabelled data and the high cost of annotation~\cite{S10,D11,H13}. According to~\cite{D11}, the two main threads of research are exploitation of cluster structure~\cite{UWBD13, DH08}, and efficient search in hypothesis space, which is the setting of our work. We are given a hypothesis class $\calH$, and the goal is to find an $h \in \calH$ that achieves a target excess generalization error, while minimizing the number of label queries. 

Three main approaches have been studied in this setting. The first and most natural one is generalized binary search~\cite{FSST97,D04,D05,N11}, which was analyzed in the realizable case by~\cite{D05} and in various limited noise settings by~\cite{K06, N11, NJC13}. While this approach has the advantage of low label complexity, it is generally inconsistent in the fully agnostic setting~\cite{DH08}. The second approach, disagreement-based active learning, is consistent in the agnostic PAC model. \cite{CAL94} provides the first disagreement-based algorithm for the realizable case.~\cite{BBL09} provides an agnostic disagreement-based algorithm, which is analyzed in~\cite{H07} using the notion of disagreement coefficient. \cite{DHM07} reduces disagreement-based active learning to passive learning;~\cite{BDL09} and~\cite{BHLZ10} further extend this work to provide practical and efficient implementations. \cite{H09, K10} give algorithms that are adaptive to the Tsybakov Noise condition. The third line of work~\cite{BBZ07, BL13, ABL14}, achieves a better label complexity than disagreement-based active learning for linear classifiers on the uniform distribution over unit sphere and logconcave distributions. However, a limitation is that their algorithm applies only to these specific settings, and it is not apparent how to apply it generally. 

Research on confidence-rated prediction has been mostly focused on empirical work, with relatively less theoretical development. Theoretical work on this topic includes KWIK learning~\cite{LLW08}, conformal prediction~\cite{SV08} and the weighted majority algorithm of~\cite{FMS04}.  The closest to our work is the recent learning-theoretic treatment by~\cite{EYW10, EYW11}. \cite{EYW10} addresses confidence-rated prediction with guaranteed error in the realizable case, and provides a predictor that abstains in the disagreement region of the version space. This predictor achieves zero error, and coverage equal to the measure of the agreement region. \cite{EYW11} shows how to extend this algorithm to the non-realizable case and obtain zero error with respect to the best hypothesis in $\calH$. Note that the predictors in~\cite{EYW10, EYW11} generally achieve less coverage than ours for the same error guarantee; in fact, if we plug them into our Algorithm~\ref{alg:labelquery}, then we recover the label complexity bounds of disagreement-based algorithms~\cite{DHM07,H09,K10}.

A formal connection between disagreement-based active learning in realizable case and perfect confidence-rated prediction (with a zero error guarantee) was established by~\cite{EYW12}. Our work can be seen as a step towards bridging these two areas, by demonstrating that active learning can be further reduced to imperfect confidence-rated prediction, with potentially higher label savings.

%% file: appendix.tex
\appendix
\section{Additional Notation and Concentration Lemmas}

We begin with some additional notation that will be used in the subsequent proofs. Recall that we define:
\begin{equation}\label{eqn:defendelta}
\sigma(n,\delta) = \frac{8}{n} (2d\ln \frac{2en}{d} + \ln \frac{24}{\delta}), 
\end{equation}
where $d$ is the VC dimension of the hypothesis class $\calH$. 

The following lemma is an immediate corollary of the multiplicative VC bound; we pick the version of the multiplicative VC bound due to~\cite{H10}.
\begin{lemma}
\label{lem:multvc}
Pick any $n \geq 1$, $\delta \in (0,1)$. Let $S_n$ be a set of $n$ iid copies of $(X,Y)$ drawn from a distribution $D$ over labelled examples. Then, the following hold with probability at least $1 - \delta$ over the choice of $S_n$:\\
(1) For all $h \in \calH$,
\begin{equation}\label{eqn:multerr}
|\err_D(h) - \err_{S_n}(h)| \leq \min (\sigma(n,\delta) + \sqrt{\sigma(n,\delta) \err_D(h)}, \sigma(n,\delta) + \sqrt{\sigma(n,\delta) \err_{S_n}(h)} )
\end{equation}
In particular, all classifiers $h$ in $\calH$ consistent with $S_n$ satisfies
\begin{equation}\label{eqn:consisterr}
\err_D(h) \leq \sigma(n, \delta)
\end{equation}
(2) For all $h, h'$ in $\calH$,
\begin{equation}\label{eqn:multerrdiff}
|(\err_D(h) - \err_D(h')) - (\err_{S_n}(h) - \err_{S_n}(h'))| \leq \sigma(n,\delta) + \min(\sqrt{\sigma(n,\delta)\rho_D(h,h')}, \sqrt{\sigma(n,\delta)\rho_{S_n}(h,h')})
\end{equation}
\begin{equation}\label{eqn:multdist}
|\rho_D(h,h') - \rho_{S_n}(h,h')| \leq \sigma(n,\delta) + \min(\sqrt{\sigma(n,\delta)\rho_D(h,h')}, \sqrt{\sigma(n,\delta)\rho_{S_n}(h,h')})
\end{equation}
Where $\sigma(n,\delta)$ is defined in Equation~\eqref{eqn:defendelta}.
\end{lemma}

We occasionally use the following (weaker) version of Lemma~\ref{lem:multvc}.
\begin{lemma}
\label{lem:addvc}
Pick any $n \geq 1$, $\delta \in (0,1)$. Let $S_n$ be a set of $n$ iid copies of $(X,Y)$. The following holds with probability at least $1 - \delta$:
(1) For all $h \in \calH$,
\begin{equation}\label{eqn:adderr}
|\err_D(h) - \err_{S_n}(h)| \leq \sqrt{4\sigma(n,\delta)}
\end{equation}
(2) For all $h, h'$ in $\calH$,
\begin{equation}\label{eqn:adderrdiff}
|(\err_D(h) - \err_D(h')) - (\err_{S_n}(h) - \err_{S_n}(h'))| \leq \sqrt{4\sigma(n,\delta)}
\end{equation}
\begin{equation}\label{eqn:adddist}
|\rho_D(h,h') - \rho_{S_n}(h,h')| \leq \sqrt{4\sigma(n,\delta)}
\end{equation}
Where $\sigma(n,\delta)$ is defined in Equation~\eqref{eqn:defendelta}.
\end{lemma}

For an unlabelled sample $U_k$, we use $\tilde{U}_k$ to denote the joint distribution over $\calX \times \calY$ induced by uniform distribution over $U_k$ and $D_{Y|X}$. We have:
\begin{lemma}
\label{lem:unldata}
If the size of $n_k$ of the unlabelled dataset $U_k$ is at least $192(\frac{256}{\epsilon_k})^2(d\ln\frac{256}{\epsilon_k} + \ln\frac{288}{\delta_k}) $, then with probability $1 - \delta_k/4$, the following conditions hold for all $h, h' \in V_k$:\\
\begin{equation}\label{eqn:unlerr}
 |\err_D(h) - \err_{\tilde{U}_k}(h) | \leq \frac{\epsilon_k}{64} 
\end{equation}
\begin{equation}\label{eqn:unlerrdiff}
| (\err_D(h) - \err_D(h')) - (\err_{\tilde{U}_k}(h) - \err_{\tilde{U}_k}(h'))| \leq \frac{\epsilon_k}{32} 
\end{equation}
\begin{equation}\label{eqn:unldist}
| \rho_D(h,h') - \rho_{\tilde{U}_k}(h,h')| \leq \frac{\epsilon_k}{64} 
\end{equation}
\end{lemma}

\begin{lemma}
\label{lem:unldatalp}
If the size of $n_k$ of the unlabelled dataset $U_k$ is at least $192(\frac{256}{\epsilon_k})^2(d\ln\frac{256}{\epsilon_k} + \ln\frac{288}{\delta_k})$, then with probability $1 - \delta_k/4$, the following hold:\\ 
(1) The outputs $\{(\xi_{k,i}, \zeta_{k,i}, \gamma_{k,i})\}_{i=1}^{n_k}$ of any confidence-rated predictor with inputs hypothesis set $V_k$, unlabelled data $U_k$, and error bound $\epsilon_k/64$ satisfy: 
\begin{equation} \label{eqn:unldisagree}
\frac{1}{n_k} \sum_{i=1}^{n_k} [I(h(x_i) \neq h'(x_i)) (1 - \gamma_{k,i})] \leq \frac{\epsilon_k}{32}; 
\end{equation}
(2) The outputs $\{(\xi_{k,i}, \zeta_{k,i}, \gamma_{k,i})\}_{i=1}^{n_k}$ of the confidence-rated predictor of Algortihm~\ref{alg:optcrp} with inputs hypothesis set $V_k$, unlabelled data $U_k$, and error bound $\epsilon_k/64$ satisfy: 
\begin{equation} \label{eqn:unluncoverage}
\phi_k \leq \bm{\Phi}_D(V_k,\frac{\epsilon_k}{128}) + \frac{\epsilon_k}{256}
\end{equation}
\end{lemma}

We use $\tilde{\Gamma}_k$ to denote the joint distribution over $\calX \times \calY$ induced by $\Gamma_k$ and $D_{Y|X}$. Denote $\gamma_k(x): \calX \to [0,1]$, where $\gamma_k(x_i) = \gamma_{k,i}$, and 0 elsewhere. Clearly, $\Gamma_k(\{x\}) = \frac{\gamma_k(x)}{n_k\phi_k}$ and $\tilde{\Gamma}_k(\{(x,y)\}) = \frac{\tilde{U}_k(\{(x,y)\}) \gamma_k(x)}{\phi_k}$. Also, Equations~\eqref{eqn:unldisagree} and~\eqref{eqn:unluncoverage} of Lemma~\ref{lem:unldatalp} can be restated as
\[ \forall h,h' \in V_k, \E_{\tilde{U}_k} [(1-\gamma_k(x)) I(h(x) \neq h'(x))] \leq \frac{\epsilon_k}{32} \]
\[ \E_{\tilde{U}_k} [\gamma_k(x)] = \phi_k \leq \bm{\Phi}_D(V_k, \frac{\epsilon_k}{128}) + \frac{\epsilon_k}{256} \]
In the realizable case, define event 
\begin{eqnarray*}
&&\text{$E_r =$ \{For all $k = 1,2,\ldots,k_0$: Equations~\eqref{eqn:unlerr},~\eqref{eqn:unlerrdiff},~\eqref{eqn:unldist},~\eqref{eqn:unldisagree},~\eqref{eqn:unluncoverage} hold for $\tilde{U}_k$ }\\&&\text{and all classifiers consistent with $S_k$ have error at most $\frac{\epsilon_k}{8\phi_k}$ with respect to $\tilde{\Gamma}_k$ \}. }
\end{eqnarray*}

\begin{fact} 
$ \P(E_r) \geq 1 - \delta $.
\end{fact}
\begin{proof}
By Equation~\eqref{eqn:consisterr} of Lemma~\ref{lem:multvc}, with probability $1-\delta_k/2$, if $h \in V_k$ is consistent with $S_k$, then
\[ \err_{\tilde{\Gamma}_k}(h) \leq \sigma(m_k, \delta_k/2) \]
Because $m_k = \frac{768\phi_k}{\epsilon_k}(d\ln\frac{768\phi_k}{\epsilon_k} + \ln\frac{48}{\delta_k})$, we have
$\err_{\tilde{\Gamma}_k}(h) \leq \epsilon_k/8\phi_k$.
The fact follows from combining the fact above with Lemma~\ref{lem:unldata} and Lemma~\ref{lem:unldatalp}, and the union bound.\\
\end{proof}

In the non-realizable case, define event
\begin{eqnarray*}
&&\text{
$E_a =$ \{For all $k = 1,2,\ldots,k_0$: Equations~\eqref{eqn:unlerr},~\eqref{eqn:unlerrdiff},~\eqref{eqn:unldist},~\eqref{eqn:unldisagree},~\eqref{eqn:unluncoverage} hold for $\tilde{U}_k$,}\\
&&\text{
 and Algorithm~\ref{alg:adaptive} succeeds with inputs hypothesis set $V = V_k$, example distribution $\Delta = \Gamma_k$,} \\
&&\text{
 labelling oracle $\calO$, target excess error $\tilde{\epsilon} = \frac{\epsilon_k}{8\phi_k}$ and target confidence $\tilde{\delta} = \frac{\delta_k}{2}$\}. }
\end{eqnarray*}

\begin{fact}
$ \P(E_a) \geq 1 - \delta $.
\end{fact}
\begin{proof}
This is an immediate consequence of Lemma~\ref{lem:unldata}, Lemma~\ref{lem:unldatalp}, Lemma~\ref{lem:ratiotype} and union bound.
\end{proof}

Recall that we assume the hypothesis space is ``dense", in the sense that $\forall r > 0$, $\sup_{h \in B_D(h^*(D),r) } \rho(h,h^*(D)) = r$. We will call this the ``denseness assumption".

\section{Proofs related to the properties of Algorithm~\ref{alg:adaptive}}

We first establish some properties of Algorithm~\ref{alg:adaptive}. The inputs to Algorithm~\ref{alg:adaptive} are a set $V$ of hypotheses of VC dimension $d$, an example distribution $\Delta$, a labeling oracle $\calO$, a target excess error $\tilde{\epsilon}$ and a target confidence $\tilde{\delta}$.

We define the event 
\[ \tilde{E} = \{\text{For all\;} j = 1,2,\ldots: \text{Equations~\eqref{eqn:multerr}-\eqref{eqn:multdist} hold for sample $S_j$ with $n = n_j$ and $\delta = \tilde{\delta}_j$ }\} \]
By union bound, $\P(\tilde{E}) \geq 1 - \sum_j \tilde{\delta}_j \geq 1 - \tilde{\delta}$. 

\begin{proof}(of Lemma~\ref{lem:ratiotype})
Assume $\tilde{E}$ happens.
For the proof of (1), define $j_{max}$ as the smallest integer $j$ such that $\enjdj \leq \tilde{\epsilon}^2/144$. Since $n_{j_{max}}$ is a power of 2,
\[
n_{j_{max}} \leq 2 \min\{n = 1,2,\ldots: \frac{8(2d\ln\frac{2en}{d} + \ln\frac{24\log n (\log n + 1)}{\delta})}{n} \leq \frac{\epsilon^2}{144} \}
\]
Thus, $n_{j_{max}} \leq 192\frac{144}{\tilde{\epsilon}^2}(d\ln\frac{144}{\tilde{\epsilon}} + \ln\frac{24}{\tilde{\delta}})$. Then in round $j_{max}$, the stopping criterion~\eqref{eqn:stopcriterion} of Algorithm~\ref{alg:adaptive} is satisified; thus, Algorithm~\ref{alg:adaptive} halts with $j_0 \leq j_{max}$.

To prove (2.1), we observe that as $h^*(\tilde{\Delta})$ is the risk minimizer in $V$, if $h$ satisfies $\err_{\tilde{\Delta}}(h) - \err_{\tilde{\Delta}}(h^*(\tilde{\Delta})) \leq \frac{\tilde{\epsilon}}{2}$, then $\err_{\tilde{\Delta}}(h) - \err_{\tilde{\Delta}}(\hat{h}_{j_0}) \leq \frac{\tilde{\epsilon}}{2}$.
By Equation~\eqref{eqn:multerrdiff} of Lemma~\ref{lem:multvc},
\begin{eqnarray*}
(\err_{S_{j_0}}(h) - \err_{S_{j_0}}(\hat{h}_{j_0}) ) & \leq & (\err_{\tilde{\Delta}}(h) - \err_{\tilde{\Delta}}(\hat{h}_{j_0}) ) + \enjzdjz + \sqrt{\enjzdjz\rho_{S_{j_0}}(h, \hat{h}_{j_0})} \\
& \leq & \frac{\tilde{\epsilon}}{2} + \enjzdjz + \sqrt{\enjzdjz\rho_{S_{j_0}}(h, \hat{h}_{j_0})}
\end{eqnarray*}
Hence $h \in V_{j_0}$.

For the proof of (2.2), note first that by (2.1), in particular, $h^*(\tilde{\Delta}) \in V_{j_0}$. 
Hence by Equation~\eqref{eqn:multerrdiff} of Lemma~\ref{lem:multvc}, and the stopping criterion Equation~\eqref{eqn:stopcriterion},
\[ (\err_{\tilde{\Delta}}(\hat{h}_{j_0}) - \err_{\tilde{\Delta}}(h^*(\tilde{\Delta})) ) - (\err_{S_{j_0}}(\hat{h}_{j_0}) - \err_{S_{j_0}}(h^*(\tilde{\Delta})) ) \leq \enjzdjz + \sqrt{\enjzdjz\rho_{S_{j_0}}(\hat{h}_{j_0}, h^*(\tilde{\Delta}))} \leq \frac{\tilde{\epsilon}}{6}\]
Thus,
\begin{equation}
\label{eqn:hathgood}
 \err_{\tilde{\Delta}}(\hat{h}_{j_0}) - \err_{\tilde{\Delta}}(h^*(\tilde{\Delta})) \leq \frac{\tilde{\epsilon}}{6} 
\end{equation}
On the other hand, if $h \in V_{j_0}$, then
\[ (\err_{\tilde{\Delta}}(h) - \err_{\tilde{\Delta}}(\hat{h}_{j_0}) ) - (\err_{S_{j_0}}(h) - \err_{S_{j_0}}(\hat{h}_{j_0}) ) \leq \enjzdjz + \sqrt{\enjzdjz\rho_{S_{j_0}}(h, \hat{h}_{j_0})} \leq \frac{\tilde{\epsilon}}{6} \]
By definition of $V_{j_0}$,
\[ (\err_{S_{j_0}}(h) - \err_{S_{j_0}}(\hat{h}_{j_0}) ) \leq \enjzdjz + \sqrt{\enjzdjz\rho_{S_{j_0}}(h, \hat{h}_{j_0})} + \frac{\tilde{\epsilon}}{2} \leq \frac{2\tilde{\epsilon}}{3}\]
Hence,
\begin{equation}\label{eqn:hgood}
\err_{\tilde{\Delta}}(h) - \err_{\tilde{\Delta}}(\hat{h}_{j_0}) \leq \frac{5\tilde{\epsilon}}{6}
\end{equation}
Combining Equations~\eqref{eqn:hathgood} and~\eqref{eqn:hgood}, we have
\[ \err_{\tilde{\Delta}}(h) - \err_{\tilde{\Delta}}(h^*(\tilde{\Delta})) \leq \tilde{\epsilon}\]
\end{proof}

\begin{proof}(of Lemma~\ref{lem:adaptivetonu})
Assume $\tilde{E}$ happens. For each $j$, by triangle inequality, we have that $\rho_{S_j}(\hat{h}_j,h) \leq \err_{S_j}(\hat{h}_j) + \err_{S_j}(h)$. If $h \in V_j$, then, by defintion of $V_j$,
\[ \err_{S_j}(h) - \err_{S_j}(\hat{h}_j) \leq \frac{\tilde{\epsilon}}{2} + \enjdj + \sqrt{\enjdj\err_{S_j}(\hat{h}_j)} + \sqrt{\enjdj\err_{S_j}(h)} \]
Using the fact that $A \leq B + C\sqrt{A} \Rightarrow A \leq 2B + C^2$,
\[ \err_{S_j}(h) \leq \tilde{\epsilon} + 2\err_{S_j}(\hat{h}_j) + 2\sqrt{\enjdj\err_{S_j}(\hat{h}_j)} + 3\enjdj \leq 3\err_{S_j}(\hat{h}_j) + 4\enjdj + \tilde{\epsilon} \]
Since 
\[ \err_{S_j}(\hat{h}_j) \leq \err_{S_j}(h^*(\tilde{\Delta})) \leq \nu^*(\tilde{\Delta}) + \sqrt{\enjdj \nu^*(\tilde{\Delta})} + \enjdj \leq 2\nu^*(\tilde{\Delta}) + 2\enjdj, \]
by the triangle inequality, we get that for all $h \in V_j$,
\begin{equation} \label{eqn:diamnu}
\rho_{S_j}(h,\hat{h}_j) \leq \err_{S_j}(h) + \err_{S_j}(\hat{h}_j) \leq 8\nu^*(\tilde{\Delta}) + 12\enjdj + \tilde{\epsilon}
\end{equation}
Now observe that for any $j$,
\begin{eqnarray*}
&&\sup_{h \in V_j} \sqrt{\enjdj\rho_{S_j}(h,\hat{h}_j)}  + \enjdj   \\
&\leq& \sup_{h \in V_j} \max(2\sqrt{\enjdj\rho_{S_j}(h,\hat{h}_j)}, 2\enjdj)  \\
&\leq& \max(2\sqrt{(8\nu^*(\tilde{\Delta}) + 12\enjdj + \tilde{\epsilon})\enjdj}, 2\enjdj)  \\
&\leq& \max(12\sqrt{2\nu^*(\tilde{\Delta})\enjdj}, \tilde{\epsilon}/6, 216\enjdj), 
\end{eqnarray*}
Where the first inequality follows from $A+B \leq 2\max(A,B)$, the second inequality follows from Equation~\eqref{eqn:diamnu}, the third inequality follows from $\sqrt{A+B} \leq \sqrt{A} + \sqrt{B}$, $A + B + C \leq 3\max(A,B,C)$ and $\sqrt{AB} \leq \max(A,B)$.
 
It can be easily seen that there exists some constant $c_1 > 0$, such that taking $j_1 = \lceil \log \left( \frac{c_1}{2}(d\ln\frac{1}{\tilde{\epsilon}} + \ln\frac{1}{\tilde{\delta}})(\frac{\nu^*(\tilde{\Delta}) + \tilde{\epsilon}}{\tilde{\epsilon}^2}) \right) \rceil$ ensures that $n_{j_1} \geq \frac{c_1}{2}(d\ln\frac{1}{\tilde{\epsilon}} + \ln\frac{1}{\tilde{\delta}})(\frac{\nu^*(\tilde{\Delta}) + \tilde{\epsilon}}{\tilde{\epsilon}^2})$; this, in turn, suffices to make
\[ \max(12\sqrt{2\nu^*(\tilde{\Delta})\enjdj}, 216\enjdj) \leq \tilde{\epsilon}/6 \]
Hence the stopping criterion $\sup_{h \in V_j} \sqrt{\enjdj\rho_{S_j}(h,\hat{h}_j)}  + \enjdj \leq \tilde{\epsilon}/6$ is satisfied in iteration $j_1$, and Algorithm~\ref{alg:adaptive} exits at iteration $j_0 \leq j_1$, which ensures that $n_{j_0} \leq n_{j_1} \leq c_1(d\ln\frac{1}{\tilde{\epsilon}} + \ln\frac{1}{\tilde{\delta}})(\frac{\nu^*(\tilde{\Delta}) + \tilde{\epsilon}}{\tilde{\epsilon}^2})$.
\end{proof}

The following lemma examines the behavior of Algorithm~\ref{alg:adaptive} under the Tsybakov Noise Condition and is crucial in the proof of Theorem~\ref{thm:labelcomplexitytnc}. We observe that even if the $(C_0, \kappa)$-Tsybakov Noise Conditions hold with respect to $D$, they do not necessarily hold with respect to $\Gamma_k$. In particular, it is not necessarily true that:
\[ \rho_{\tilde{\Gamma}_k}(h,h^*(D)) \leq C_0 (\err_{\tilde{\Gamma}_k}(h) - \err_{\tilde{\Gamma}_k}(h^*(D)))^{\frac{1}{\kappa}}, \forall h \in V_k \]
However, we show that  an ``approximate'' Tsybakov Noise Condition with a significantly larger ``$C_0$", namely Condition~\eqref{eqn:approxtnc} is met by $\tilde{\Gamma}_k$ and $V_k$, with $C = \max(8C_0,4) \phi_k^{\frac{1}{\kappa}-1}$ and $\tilde{h} = h^*(D)$. In the Lemma below, we carefully track the dependence of the number of our label queries on $C$, since $C= \max(8C_0,4) \phi_k^{\frac{1}{\kappa}-1}$ can be $\omega(1)$ in our particular application.

\begin{lemma}\label{lem:adaptivetotnc}
Suppose we run Algorithm~\ref{alg:adaptive} with inputs hypothesis set $V$, example distribution $\tilde{\Delta}$, labelling oracle $\calO$, excess generalization error $\tilde{\epsilon}$ and confidence $\tilde{\delta}$. Then there exists some absolute constant $c_2 > 0$ (independent of $C$) such that the following holds. Suppose there exist $C>0$ and a classifier $\tilde{h} \in V$, such that
\begin{equation}\label{eqn:approxtnc} 
\forall h \in V, \rho_{\tilde{\Delta}}(h,\tilde{h}) \leq C \max(\tilde{\epsilon},\err_{\tilde{\Delta}}(h) - \err_{\tilde{\Delta}}(\tilde{h}))^{\frac{1}{\kappa}}, 
\end{equation}
where $\tilde{\epsilon}$ is the target exccess error parameter in Algorithm~\ref{alg:adaptive}. Then, on the event that Algorithm~\ref{alg:adaptive} succeeds, 
\[ n_{j_0} \leq c_2 \max( (d \ln\frac{1}{\tilde{\epsilon}}+ \ln\frac{1}{\tilde{\delta}})\tilde{\epsilon}^{-1}, (d \ln(C\tilde{\epsilon}^{\frac{1}{\kappa}-2}) + \ln\frac{1}{\tilde{\delta}})C\tilde{\epsilon}^{\frac{1}{\kappa}-2} ) \]
\end{lemma}

Observe that Condition~\eqref{eqn:approxtnc}, the approximate Tsybakov Noise Condition in the statement of Lemma~\ref{lem:adaptivetotnc}, is with respect to $\tilde{h}$, which is not necessarily the true risk minimizer in $V$ with respect to $\tilde{\Delta}$. We therefore prove Lemma~\ref{lem:adaptivetotnc} in three steps; first, in Lemma~\ref{lem:ermtnc}, we analyze the difference $\err_{\tilde{\Delta}}(\hat{h}) - \err_{\tilde{\Delta}}(\tilde{h})$, where $\hat{h}$ is the empirical risk minimizer. Then, in Lemma~\ref{lem:vstnc}, we bound the difference $\err_{\tilde{\Delta}}(h) - \err_{\tilde{\Delta}}(\tilde{h})$ for any $h \in V_j$ for some $j$. Finally, we combine these two lemmas to provide sample complexity bounds for the $V_{j_0}$ output by Algorithm~\ref{alg:adaptive}. 

\begin{proof}(of Lemma~\ref{lem:adaptivetotnc})
Assume the event $\tilde{E}$ happens. Then,

Consider iteration $j$, by Lemma~\ref{lem:vstnc}, if $h \in V_j$, then
\begin{equation}\label{eqn:diamtnc}
\rho_{\tilde{\Delta}}(h,\hat{h}_j) \leq \rho_{\tilde{\Delta}}(h,\tilde{h}) + \rho_{\tilde{\Delta}}(\hat{h}_j,\tilde{h}) \leq \max(2C(36\tilde{\epsilon})^{\frac{1}{\kappa}}, 2C(52\enjdj)^{\frac{1}{\kappa}}, 2C(6400C\enjdj)^{\frac{1}{2\kappa-1}}).
\end{equation}

We can write:
\begin{eqnarray*}
\sup_{h \in V_j} \enjdj + \sqrt{\enjdj \rho_{S_j}(h,\hat{h}_j)}  &\leq& \sup_{h \in V_j} 3\enjdj + \sqrt{2\enjdj \rho_{\tilde{\Delta}}(h,\hat{h}_j)}  \\
&\leq& \sup_{h \in V_j} \max(6\enjdj, 2\sqrt{2\enjdj \rho_{\tilde{\Delta}}(h,\hat{h}_j)}),
\end{eqnarray*}
where the first inequality follows from Equation~\eqref{eqn:distconc} and the second inequality follows $A + B \leq 2 \max(A, B)$. We can further use Equation~\eqref{eqn:diamtnc} to show that this is at most:
\begin{eqnarray*}
&\leq& \max(6\enjdj, (16C\enjdj)^{\frac{1}{2}}(36\tilde{\epsilon})^{\frac{1}{2\kappa}}, (16C\enjdj)^{\frac{1}{2}}(52\enjdj)^{\frac{1}{2\kappa}}, (6400C\enjdj)^{\frac{\kappa}{2\kappa-1}} )   \\
&\leq& \max(6\enjdj, \tilde{\epsilon}/6, (6400C\enjdj)^{\frac{\kappa}{2\kappa-1}} )  \\
\end{eqnarray*}
Here the last inequality follows from the fact that $(16C\enjdj)^{\frac{1}{2}}(36\tilde{\epsilon})^{\frac{1}{2\kappa}} \leq \max((3456C\enjdj)^{\frac{\kappa}{2\kappa-1}},\tilde{\epsilon}/6)$ and $(16C\enjdj)^{\frac{1}{2}}(52\enjdj)^{\frac{1}{2\kappa}} \leq \max((144C\enjdj)^{\frac{\kappa}{2\kappa-1}},6\enjdj)$, since $A^{\frac{2\kappa-1}{2\kappa}} B^{\frac{1}{2\kappa}} \leq \max(A,B)$.

It can be easily seen that there exists $c_2 > 0$, such that taking $j_1 = \lceil \log \frac{c_2}{2} (d \ln\frac{\max(C,1)}{\tilde{\epsilon}}+ \ln\frac{1}{\tilde{\delta}})(C\tilde{\epsilon}^{\frac{1}{\kappa}-2} + \tilde{\epsilon}^{-1}) \rceil$, so that $n_j \geq \frac{c_2}{2} (d \ln\frac{\max(C,1)}{\tilde{\epsilon}}+ \ln\frac{1}{\tilde{\delta}})(C\tilde{\epsilon}^{\frac{1}{\kappa}-2} + \tilde{\epsilon}^{-1})$ suffices to make 
\[ \max(6\enjdj,  (6400C\enjdj)^{\frac{\kappa}{2\kappa-1}} ) \leq \tilde{\epsilon}/6 \]
Hence the stopping criterion $\sup_{h \in V_j} \sqrt{\enjdj\rho_{S_j}(h,\hat{h}_j)}  + \enjdj \leq \tilde{\epsilon}/6$ is satisfied in iteration $j_1$. Thus the number of the exit iteration $j_0$ satisfies $j_0 \leq j_1$, and $n_{j_0} \leq n_{j_1} \leq c_2 \max( (d \ln\frac{1}{\tilde{\epsilon}}+ \ln\frac{1}{\tilde{\delta}})\tilde{\epsilon}^{-1}, (d \ln(C\tilde{\epsilon}^{\frac{1}{\kappa}-2}) + \ln\frac{1}{\tilde{\delta}})C\tilde{\epsilon}^{\frac{1}{\kappa}-2} )$.

\end{proof}

\begin{lemma}\label{lem:ermtnc}
Suppose there exist $C>0$ and a classifier $\tilde{h} \in V$, such that Equation~\eqref{eqn:approxtnc} holds. Suppose we draw a set $S$ of $n$ examples, denote the empirical risk minimizer over $S$ as $\hat{h}$, then with probability $1-\delta$:
\[ \err_{\tilde{\Delta}}(\hat{h}) - \err_{\tilde{\Delta}}(\tilde{h}) \leq \max(2\endelta, (4C\endelta)^{\frac{\kappa}{2\kappa-1}},2\tilde{\epsilon}) \]
\[ \rho_{\tilde{\Delta}}(\hat{h}, \tilde{h}) \leq \max(C(2\endelta)^{\frac{1}{\kappa}}, C(4C\endelta)^{\frac{1}{2\kappa-1}}, C(2\tilde{\epsilon})^{\frac{1}{\kappa}}) \]
\end{lemma}
\begin{proof}
By Lemma~\ref{lem:multvc}, with probability $1-\delta$, Equation~\eqref{eqn:multerrdiff} holds. Assume this happens.
\begin{eqnarray}
 && \err_{\tilde{\Delta}}(\hat{h}) - \err_{\tilde{\Delta}}(\tilde{h}) \nonumber \\
 &\leq& \endelta + \sqrt{\endelta \rho_{\tilde{\Delta}}(\hat{h},\tilde{h})} \nonumber \\
 &\leq& 2\max(\endelta, \sqrt{\endelta C(\err_{\tilde{\Delta}}(h) - \err_{\tilde{\Delta}}(\tilde{h})^{\frac{1}{\kappa}})}, \sqrt{\endelta C\tilde{\epsilon}^{\frac{1}{\kappa}}}) \nonumber \\
 &\leq& \max(2\endelta, (4C\endelta)^{\frac{\kappa}{2\kappa-1}},2\tilde{\epsilon}) \nonumber
\end{eqnarray}
Where the first inequality is by Equation~\eqref{eqn:multerrdiff} of Lemma~\ref{lem:multvc}; the second inequality follow from Equation~\eqref{eqn:approxtnc} and $A+B \leq 2\max(A,B)$. The third inequality follows from $2\sqrt{\endelta C\tilde{\epsilon}^{\frac{1}{\kappa}}} \leq \max(2(C\endelta)^{\frac{\kappa}{2\kappa-1}}, 2\tilde{\epsilon})$, since $A^{\frac{2\kappa-1}{2\kappa}} B^{\frac{1}{2\kappa}} \leq \max(A,B)$.
As a consequence, by Equation~\eqref{eqn:approxtnc},
\begin{equation} 
\rho_{\tilde{\Delta}}(\hat{h}, \tilde{h}) \leq \max(C(2\endelta)^{\frac{1}{\kappa}}, C(4C\endelta)^{\frac{1}{2\kappa-1}}, C(2\tilde{\epsilon})^{\frac{1}{\kappa}}) \nonumber
\end{equation}
\end{proof}

\begin{lemma}\label{lem:vstnc}
Suppose there exist a $C>0$ and a classifier $\tilde{h} \in V$ such that Equation~\eqref{eqn:approxtnc} holds. Suppose we draw a set $S$ of $n$ iid examples, and let $\hat{h}$ denote the empirical risk minimizer over $S$. Moreover, we define: 
\[ \tilde{V} = \Big{\{}h \in V: \err_S(h) \leq \err_S(\hat{h}) +  \frac{\tilde{\epsilon}}{2} + \endelta + \sqrt{\endelta \rho_S(h, \hat{h})} \Big{\}} \]
 then with probability $1-\delta$, for all $h \in \tilde{V}$,
\[ \err_{\tilde{\Delta}}(h) - \err_{\tilde{\Delta}}(\tilde{h}) \leq \max(52\endelta, 36\tilde{\epsilon}, (6400C\endelta)^{\frac{\kappa}{2\kappa-1}}) \]
\[ \rho_{\tilde{\Delta}}(h,\tilde{h}) \leq \max(C(36\tilde{\epsilon})^{\frac{1}{\kappa}}, C(52\endelta)^{\frac{1}{\kappa}}, C(6400C\endelta)^{\frac{1}{2\kappa-1}})\]
\end{lemma}

\begin{proof}
First, by Lemma~\ref{lem:ermtnc},
\begin{equation} \label{eqn:hathgoodtnc}
\err_{\tilde{\Delta}}(\hat{h}) - \err_{\tilde{\Delta}}(\tilde{h}) \leq \max(2\endelta, (4C\endelta)^{\frac{\kappa}{2\kappa-1}},2\tilde{\epsilon}) 
\end{equation}
\begin{equation} \label{eqn:hathclosetnc}
\rho_{\tilde{\Delta}}(\hat{h}, \tilde{h}) \leq \max(C(2\endelta)^{\frac{1}{\kappa}}, C(4C\endelta)^{\frac{1}{2\kappa-1}}, C(2\tilde{\epsilon})^{\frac{1}{\kappa}})
\end{equation}
Next, if $h \in \tilde{V}$, then
\[ \err_S(h) - \err_S(\hat{h}) \leq \endelta + \sqrt{\endelta \rho_S(h,\hat{h})} + \frac{\tilde{\epsilon}}{2} \]
Combining it with Equation~\eqref{eqn:multerrdiff} of Lemma~\ref{lem:multvc}: $\err_{\tilde{\Delta}}(h) - \err_{\tilde{\Delta}}(\hat{h}) \leq \err_S(h) - \err_S(\hat{h}) + \sqrt{\endelta\rho_S(h,\hat{h})} + \endelta$, we get
\[ \err_{\tilde{\Delta}}(h) - \err_{\tilde{\Delta}}(\hat{h}) \leq 2\endelta + 2\sqrt{\endelta\rho_S(h,\hat{h})} + \frac{\tilde{\epsilon}}{2} \]
By Equation~\eqref{eqn:multdist} of Lemma~\ref{lem:multvc},
\begin{equation} \label{eqn:distconc}
\rho_S(h,\hat{h}) \leq \rho_{\tilde{\Delta}}(h,\hat{h}) + \sqrt{\endelta\rho_{\tilde{\Delta}}(h,\hat{h})} + \endelta \leq 2\rho_{\tilde{\Delta}}(h,\hat{h}) + 2\endelta 
\end{equation}
Therefore,
\begin{equation} \label{eqn:hgoodtnc}
\err_{\tilde{\Delta}}(h) - \err_{\tilde{\Delta}}(\hat{h}) \leq 5\endelta + 3\sqrt{\endelta\rho_{\tilde{\Delta}}(h,\hat{h})} + \frac{\tilde{\epsilon}}{2} 
\end{equation}
Hence
\begin{eqnarray*}
&&\err_{\tilde{\Delta}}(h) - \err_{\tilde{\Delta}}(\tilde{h}) \\
&=& (\err_{\tilde{\Delta}}(h) - \err_{\tilde{\Delta}}(\hat{h})) + (\err_{\tilde{\Delta}}(\hat{h}) - \err_{\tilde{\Delta}}(\tilde{h})) \\
&\leq& (4C\endelta)^{\frac{\kappa}{2\kappa-1}} + 7\endelta + 3\tilde{\epsilon} + 3\sqrt{\endelta \rho_{\tilde{\Delta}}(h,\hat{h})} \\ 
&\leq& (4C\endelta)^{\frac{\kappa}{2\kappa-1}} + 7\endelta + 3\tilde{\epsilon} + 3\sqrt{\endelta\rho_{\tilde{\Delta}}(h,\tilde{h})} + 3\sqrt{\endelta\rho_{\tilde{\Delta}}(\tilde{h},\hat{h})} 
\end{eqnarray*}
Here the first inequality follows from Equations~\eqref{eqn:hathgoodtnc} and~\eqref{eqn:hgoodtnc} and $\max(A,B,C) \leq A+B+C$, and the second inequality follows from triangle inequality and $\sqrt{A+B} \leq \sqrt{A} + \sqrt{B}$.

From Equation~\eqref{eqn:hathclosetnc}, $\endelta \rho_{\tilde{\Delta}}(\hat{h}, \tilde{h})$ is at most:
\begin{eqnarray*}
& \leq & C \endelta \cdot ( (2 \tilde{\epsilon})^{1/\kappa} + (2 \endelta)^{1/\kappa} + (4 C \endelta)^{1/(2 \kappa - 1)} ) \\
& \leq & (4 C \endelta)^{2 \kappa / (2 \kappa - 1)} + C \endelta (  (2 \tilde{\epsilon})^{1/\kappa} + (2 \endelta)^{1/\kappa}) \\
& \leq & (4 C \endelta)^{2 \kappa / (2 \kappa - 1)} + \max( 4 \tilde{\epsilon}^2, (C \endelta)^{2 \kappa/(2 \kappa - 1)} ) + \max( 4 \endelta^2, (C \endelta)^{2 \kappa/(2 \kappa - 1)}),
\end{eqnarray*}
where the first step follows from Equation~\eqref{eqn:hathclosetnc}, the second step from algebra, and the third step from using the fact that $A^{\frac{2 \kappa - 1}{\kappa}} B^{\frac{1}{\kappa}} \leq \max(A^2, B^2)$. Plugging this in to the previous equation, and using $\max(A, B) \leq A + B$ and $\sqrt{A + B} \leq \sqrt{A} + \sqrt{B}$, we get that:
\begin{eqnarray*}
\err_{\tilde{\Delta}}(h) - \err_{\tilde{\Delta}}(\tilde{h}) & \leq & 10 (4 C \endelta)^{\kappa/(2\kappa - 1)} + 9 \tilde{\epsilon} + 13 \endelta + 3 \sqrt{\endelta \rho_{\tilde{\Delta}}(h, \tilde{h})} 
\end{eqnarray*}
Combining this with the fact that $A + B + C + D \leq 4 \max(A, B, C, D)$, we get that this is at most:
\begin{eqnarray*}
& \leq & \max( 40 (4 C \endelta)^{\kappa/(2\kappa - 1)},  36 \tilde{\epsilon}, 52 \endelta,  12 \sqrt{\endelta \rho_{\tilde{\Delta}}(h, \tilde{h})}) \end{eqnarray*}
Combining this with Condition~\eqref{eqn:approxtnc}, we get that this is at most:
\begin{eqnarray*}
\max( 40 (4 C \endelta)^{\kappa/(2\kappa - 1)},  36 \tilde{\epsilon}, 52 \endelta, 12 \sqrt{ C \endelta \tilde{\epsilon}^{1/\kappa} }, 12 \sqrt{ C \endelta (\err_{\tilde{\Delta}}(h) - \err_{\tilde{\Delta}}(\tilde{h}))^{1/\kappa}})
\end{eqnarray*}

Using $A^{(2 \kappa - 1)/2\kappa} B^{1/2 \kappa} \leq \max(A, B)$, we get that $\sqrt{C \endelta \tilde{\epsilon}^{1/\kappa}} \leq \max(\tilde{\epsilon}, (C \endelta)^{\kappa/(2 \kappa - 1)})$. Also note $\err_{\tilde{\Delta}}(h) - \err_{\tilde{\Delta}}(\tilde{h}) \leq 12 \sqrt{ C \endelta (\err_{\tilde{\Delta}}(h) - \err_{\tilde{\Delta}}(\tilde{h}))^{1/\kappa}}$ implies $\err_{\tilde{\Delta}}(h) - \err_{\tilde{\Delta}}(\tilde{h}) \leq (144 C \endelta)^{\kappa/(2\kappa - 1)}$. Thus we have
\[ \err_{\tilde{\Delta}}(h) - \err_{\tilde{\Delta}}(\tilde{h}) \leq \max(36\tilde{\epsilon}, 52\endelta, (6400C\endelta)^{\frac{\kappa}{2\kappa-1}}) \]

Invoking~\eqref{eqn:approxtnc} again, we have that:
\[ \rho_{\tilde{\Delta}}(h,\tilde{h}) \leq \max(C(36\tilde{\epsilon})^{\frac{1}{\kappa}}, C(52\endelta)^{\frac{1}{\kappa}}, C(6400C\endelta)^{\frac{1}{2\kappa-1}})\]

\end{proof}

\section{Remaining Proofs from Section~\ref{sec:alg}}

\begin{proof}(Of Lemma~\ref{lem:vsrealizable})
Assuming $E_r$ happens, we prove the lemma by induction.\\
\textbf{Base Case:} For $k = 1$, clearly $h^*(D) \in V_1 = \calH$.\\
\textbf{Inductive Case:} Assume $h^*(D) \in V_k$. As we are in the realizable case, $h^*(D)$ is consistent with the examples $S_k$ drawn in Step 8 of Algorithm~\ref{alg:labelquery}; thus $h^*(D) \in V_{k+1}$. The lemma follows.
\end{proof}

\begin{proof}(Of Lemma~\ref{lem:vsnonrealizable})
We use $\tilde{h}_{k} = \argmin_{h \in V_k} \err_{\tilde{\Gamma}_k}(h)$ to denote the optimal classifier in $V_k$ with respect to the distribution $\tilde{\Gamma}_k$. Assuming $E_a$ happens, we prove the lemma by induction.\\
\textbf{Base Case:} For $k = 1$, clearly $h^*(D) \in V_1 = \calH$.\\
\textbf{Inductive Case:} Assume $h^* \in V_k$. In order to show the inductive case, our goal is to show that:
\begin{equation} \label{eqn:cons}
 \P_{\tilde{\Gamma}_k}(h^*(D)(x) \neq y) - \P_{\tilde{\Gamma}_k}(\tilde{h}_k(x) \neq y) \leq \frac{\epsilon_k}{16\phi_k} 
\end{equation}
If~\eqref{eqn:cons} holds, then, by (2.1) of Lemma~\ref{lem:ratiotype}, we know that if Algorithm~\ref{alg:adaptive} succeeds when called in iteration $k$ of Algorithm~\ref{alg:labelquery}, then, it is guaranteed that $h^* \in V_{k+1}$. 

We therefore focus on showing~\eqref{eqn:cons}. First, from Equation~\eqref{eqn:unlerrdiff} of Lemma~\ref{lem:unldata}, we have:
\[ (\err_{\tilde{U}_k}(h^*(D)) - \err_{\tilde{U}_k}(\tilde{h}_{k})) - (\err_D(h^*(D)) - \err_D(\tilde{h}_{k})) \leq \frac{\epsilon_k}{32}\]
As $\err_D(h^*(D)) \leq \err_D(\tilde{h}_{k})$, we get:
\begin{equation} \label{eqn:dnotbad}
\err_{\tilde{U}_k}(h^*(D)) \leq \err_{\tilde{U}_k}(\tilde{h}_{k}) + \frac{\epsilon_k}{32}
\end{equation}
On the other hand, by Equation~\eqref{eqn:unldisagree} of Lemma~\ref{lem:unldatalp} and triangle inequality,
\begin{eqnarray} \label{eqn:gammanotbad}
&& \E_{\tilde{U}_k} [I(\tilde{h}_{k}(x) \neq y)(1-\gamma_k(x))] - \E_{\tilde{U}_k} [I(h^*(D)(x) \neq y)(1-\gamma_k(x))] \\
&\leq & \E_{\tilde{U}_k} [I(h^*(D)(x) \neq \tilde{h}_{k}(x))(1-\gamma_k(x)) ] \leq \frac{\epsilon_k}{32}
\end{eqnarray}
Combining Equations~\eqref{eqn:dnotbad} and~\eqref{eqn:gammanotbad}, we get:
\begin{eqnarray*}
\E_{\tilde{U}_k} [I(h^*(D)(x) \neq y) \gamma_k(x)] & = & \err_{\tilde{U}_k}(h^*(D)(x)) - \E_{\tilde{U}_k}[ I(h^*(D)(x) \neq y) (1 - \gamma_k(x))] \\
& \leq & \err_{\tilde{U}_k}(\tilde{h}_k(x)) + \epsilon_k/32 - \E_{\tilde{U}_k}[ I(h^*(D)(x) \neq y) (1 - \gamma_k(x))] \\
& \leq & \E_{\tilde{U}_k} [I(\tilde{h}_k(x) \neq y) \gamma_k(x)] + \E_{\tilde{U}_k} [I(\tilde{h}(x) \neq y) (1 - \gamma_k(x))] + \epsilon_k/32 \\
&& - \E_{\tilde{U}_k}[ I(h^*(D)(x) \neq y) (1 - \gamma_k(x))] \\
& \leq & \E_{\tilde{U}_k} [I(\tilde{h}_k(x) \neq y) \gamma_k(x)] + \epsilon_k/16
\end{eqnarray*}

Dividing both sides by $\phi_k$, we get:
\[ \P_{\tilde{\Gamma}_k}(h^*(D)(x) \neq y) - \P_{\tilde{\Gamma}_k}(\tilde{h}_k(x) \neq y) \leq \frac{\epsilon_k}{16\phi_k}, \]
from which the lemma follows.
\end{proof}

\begin{proof}(of Lemma~\ref{lem:errdecrrealizable})
Assuming $E_r$ happens, we prove the lemma by induction.\\
\textbf{Base Case:} For $k = 1$, clearly $\err_D(h) \leq 1 \leq \epsilon_1 = \epsilon 2^{k_0}, \forall h \in V_1 = \calH$.\\
\textbf{Inductive Case:} Note that $\forall h, h' \in V_{k+1} \subseteq V_k$, by Equation~\eqref{eqn:unldisagree} of Lemma~\ref{lem:unldatalp}, we have:
\[ \E_{\tilde{U}_k}[I(h(x) \neq h'(x))(1-\gamma_k(x))] \leq \frac{\epsilon_k}{8} \]
By the proof of Lemma~\ref{lem:vsrealizable}, $h^*(D) \in V_{k+1}$ on event $E_r$, thus $\forall h \in V_{k+1}$,
\begin{equation}\label{eqn:vk+1uncrealizable}
 \E_{\tilde{U}_k}[I(h(x) \neq h^*(D)(x))(1-\gamma_k(x))] \leq \frac{\epsilon_k}{8} 
\end{equation}
Since any $h \in V_{k+1}$, $h$ is consistent with $S_k$ of size $m_k = \frac{768\phi_k}{\epsilon_k}(d\ln\frac{768\phi_k}{\epsilon_k} + \ln\frac{48}{\delta_k})$, we have that for all $h \in V_{k+1}$,
\[ \P_{\tilde{\Gamma}_k} (h(x) \neq h^*(D)(x)) \leq \frac{\epsilon_k}{8\phi_k} \]
That is,
\[ \E_{\tilde{U}_k} [I(h(x) \neq h^*(D)(x)) \gamma_k(x)] \leq \frac{\epsilon_k}{8} \]
Combining this with Equation~\eqref{eqn:vk+1uncrealizable} above,
\[ \P_{\tilde{U}_k}(h(x) \neq h^*(D)(x)) \leq \frac{\epsilon_k}{4} \]
By Equation~\eqref{eqn:unlerr} of Lemma~\ref{lem:unldata},
\[ \P_D(h(x) \neq h^*(D)(x)) \leq \frac{\epsilon_k}{2} = \epsilon_{k+1} \]
The lemma follows.
\end{proof}

\begin{proof}(of Lemma~\ref{lem:errdecrnonrealizable})
Assuming $E_a$ happens, we prove the lemma by induction.\\
\textbf{Base Case:} For $k = 1$, clearly $\err_D(h) - \err_D(h^*(D)) \leq 1 \leq \epsilon_1 = \epsilon 2^{k_0}, \forall h \in V_1 = \calH$.\\
\textbf{Inductive Case:} Note that $\forall h,h' \in V_{k+1} \subseteq V_k$, by Equation~\eqref{eqn:unldisagree} of Lemma~\ref{lem:unldatalp}, 
\[ \E_{\tilde{U}_k} [I(h(x) \neq y)(1-\gamma_k(x))] - \E_{\tilde{U}_k} [I(h'(D)(x) \neq y)(1-\gamma_k(x))] \leq \E_{\tilde{U}_k} [I(h(x) \neq h'(D)(x))(1-\gamma_k(x))] \leq \frac{\epsilon_k}{8} \]
From Lemma~\ref{lem:vsnonrealizable}, $h^*(D) \in V_k$ whenever the event $E_a$ happens. Thus $\forall h \in V_{k+1}$,
\begin{equation}\label{eqn:vk+1uncnonrealizable} \E_{\tilde{U}_k}I(h(x) \neq y)(1-\gamma_k(x)) - \E_{\tilde{U}_k}I(h^*(D)(x) \neq y)(1-\gamma_k(x)) \leq \frac{\epsilon_k}{8} 
\end{equation}
On the other hand, if Algorithm~\ref{alg:adaptive} succeeds with target excess error $\frac{\epsilon_k}{8\phi_k}$, by item(2.2) of Lemma~\ref{lem:ratiotype}, for any $h \in V_{k+1}$,
\[ \P_{\tilde{\Gamma}_k}(h(x) \neq y) - \min_{h \in V_k} \P_{\tilde{\Gamma}_k}(h(x) \neq y) 
\leq \frac{\epsilon_k}{8\phi_k} \]
Moreover, as $h^*(D) \in V_k$ from Lemma~\ref{lem:vsnonrealizable},
\[ \P_{\tilde{\Gamma}_k}(h(x) \neq y) - \P_{\tilde{\Gamma}_k}(h^*(D)(x) \neq y) \leq \frac{\epsilon_k}{8\phi_k} \]
In other words,
\[ \E_{\tilde{U}_k}[I(h(x) \neq y) \gamma_k(x)] - \E_{\tilde{U}_k} [I(h^*(D)(x) \neq y) \gamma_k(x)] \leq \frac{\epsilon_k}{8} \]
Combining this with Equation~\eqref{eqn:vk+1uncnonrealizable}, we get that for all $h \in V_{k+1}$,
\[ \P_{\tilde{U}_k}(h(x) \neq y) - \P_{\tilde{U}_k}(h^*(D)(x) \neq y) \leq \frac{\epsilon_k}{4} \]
Finally, combining this with Equation~\eqref{eqn:unlerrdiff} of Lemma~\ref{lem:unldata}, we have that:
\[ \P_D(h(x) \neq y) - \P_D(h^*(D)(x) \neq y) \leq \frac{\epsilon_k}{2} = \epsilon_{k+1}\]
The lemma follows.
\end{proof}

\begin{proof}(of Theorem~\ref{thm:optcrperr})
In the realizable case, We observe that for example $z_i$, $\zeta_i = \P(P(z_i) = -1)$, $\xi_i = \P(P(z_i) = 1)$, and $\gamma_i = \P(P(z_i) = 0)$. Suppose $h^* \in \calH$ is the true hypothesis which has $0$ error with respect to the data distribution. By the realizability assumption, $h^* \in V$. Moreover, $\P_U(P(x) \neq h^*(x), P(x) \neq 0) = \frac{1}{m}(\sum_{i: h^*(z_i) = +1} \zeta_i + \sum_{i: h^*(z_i) = -1} \xi_i) \leq \errguar$ by Algorithm~\ref{alg:optcrp}. \\
In the non-realizable case, we still have $\P_{x \sim U}(h^*(x) \neq P(x), P(x) \neq 0) \leq \errguar$, hence by triangle inequality, $\P_{x \sim U}(P(x) \neq x, P(x) \neq 0) - \P_{x \sim U}(h^*(x) \neq y, P(x) \neq 0) \leq \errguar$. Thus
\[ \P_{x \sim U}(P(x) \neq y, P(x) \neq 0) \leq \P_{x \sim U}(h^*(x) \neq y) + \errguar \]
\end{proof}

\begin{proof}(of Theorem~\ref{thm:optcrpcov})
Suppose $P'$ assigns probabilities $\{ [\xi'_i, \zeta'_i, \gamma'_i], i = 1, \ldots, m \}$ to the unlabelled examples $z_i$, and suppose for the sake of contradiction that $\sum_{i=1}^m \xi'_i + \zeta'_i > \sum_{i=1}^m \xi_i + \zeta_i$. Then, $\{ \xi'_i, \zeta'_i, \gamma'_i \}$'s cannot satisfy the LP in Algorithm~\ref{alg:optcrp}, and thus there exists some $h' \in V$ for which constraint~\eqref{eqn:errconstraint} is violated. The true hypothesis that generates the data could be any $h \in V$; if this true hypothesis is $h'$, then $\P_{x \sim U}(P'(x) \neq h'(x), P'(x) \neq 0) > \delta$. 
\end{proof}

\section{Proofs from Section~\ref{sec:performgrt}}
\begin{proof}(of Theorem~\ref{thm:labelcomplexity})\\
(1) In the realizable case, suppose that event $E_r$ happens. Then from Equation~\eqref{eqn:unluncoverage} of Lemma~\ref{lem:unldatalp}, while running Algorithm~\ref{alg:optcrp}, we have that:
\[ \phi_k \leq \bm{\Phi}_D(V_k, \frac{\epsilon_k}{128}) + \frac{\epsilon_k}{256} \leq \bm{\Phi}_D(B_D(h^*,\epsilon_k), \frac{\epsilon_k}{128}) + \frac{\epsilon_k}{256} \leq \bm{\Phi}_D(B_D(h^*,\epsilon_k), \frac{\epsilon_k}{256}) = \phi(\epsilon_k,\frac{\epsilon_k}{256}) \]
where the second inequality follows from the fact that $V_k \subseteq B_D(h^*(D), \epsilon_k)$, and third inequality follows from Lemma~\ref{lem:uncoveragewrterrguar} and denseness assuption.\\
Thus, there exists $c_3 > 0$ such that, in round $k$,
\[m_k =  (d \ln\frac{768\phi_k}{\epsilon_k} + \ln\frac{48}{\delta_k}) \frac{768\phi_k}{\epsilon_k} \leq c_3(d\ln\frac{\phi(\epsilon_k,\epsilon_k/256)}{\epsilon_k} + \ln(\frac{k_0 - k + 1}{\delta})) \frac{\phi(\epsilon_k,\epsilon_k/256)}{\epsilon_k} \]
Hence the total number of labels queried by Algorithm~\ref{alg:labelquery} is at most
\[\sum_{k=1}^{\lceil \log\frac{1}{\epsilon} \rceil} m_k \leq c_3 \sum_{k=1}^{\lceil \log\frac{1}{\epsilon} \rceil} (d\ln\frac{\phi(\epsilon_k,\epsilon_k/256)}{\epsilon_k} + \ln(\frac{k_0 - k + 1}{\delta})) \frac{\phi(\epsilon_k,\epsilon_k/256)}{\epsilon_k} \]

(2) In the agnostic case, suppose the event $E_a$ happens.\\
First, given $E_a$, from Equation~\eqref{eqn:unluncoverage} of Lemma~\ref{lem:unldatalp} when running Algorithm~\ref{alg:optcrp},
\begin{equation}\label{eqn:ubphiknonrealizable}
 \phi_k \leq \bm{\Phi}_D(V_k, \frac{\epsilon_k}{128}) + \frac{\epsilon_k}{256} \leq \bm{\Phi}_D(B_D(h^*, 2\nu^*(D) + \epsilon_k), \frac{\epsilon_k}{256}) = \phi(2\nu^*(D) + \epsilon_k, \frac{\epsilon_k}{256})
\end{equation}
where the second inequality follows from the fact that $V_k \subseteq B_D(h^*(D), 2\nu^*(D) + \epsilon_k)$ and the third inequality follows from Lemma~\ref{lem:uncoveragewrterrguar} and denseness assumption.
\\
Second, recall that $\tilde{h}_{k} = \argmin_{h \in V_k} \err_{\tilde{\Gamma}_k}(h)$,
\begin{eqnarray*}
\err_{\tilde{\Gamma}_k}(\tilde{h}_k) &=& \min_{h \in V_k} \err_{\tilde{\Gamma}_k}(h) \\
&\leq& \err_{\tilde{\Gamma}_k}(h^*(D)) \\
&=& \frac{\E_{\tilde{U}_k} [I(h^*(D)(x) \neq y)\gamma_k(x)] }{\phi_k} \\
&\leq& \frac{\P_{\tilde{U}_k}(h^*(D)(x) \neq y)}{\phi_k} \\
&\leq& \frac{\nu^*(D) + \epsilon_k / 64}{\phi_k} \\
\end{eqnarray*}
Here the first inequality follows from the suboptimality of $h^*(D)$ under distribution $\tilde{\Gamma}_k$, the second inequality follows from $\gamma_k(x) \leq 1$, and the third inequality follows from Equation~\eqref{eqn:unlerr}.\\
Thus, conditioned on $E_a$, in iteration $k$, Algorithm~\ref{alg:adaptive} succeeds by Lemma~\ref{lem:adaptivetonu}, and there exists a constant $c_4 > 0$ such that the number of labels queried is
\begin{eqnarray*}
&& m_k \leq c_1 \frac{\frac{\epsilon_k}{8\phi_k} + \err_{\tilde{\Gamma}_k}(\tilde{h}_k)}{(\frac{\epsilon_k}{8\phi_k})^2} (d\ln\frac{1}{\frac{\epsilon_k}{8\phi_k}}+\ln\frac{2}{\delta_k})  \\
&\leq& c_4( d\ln\frac{\phi(2\nu^*(D) + \epsilon_k,\epsilon_k/256)}{\epsilon_k} + \ln(\frac{k_0 - k + 1}{\delta}) ) \frac{\phi(2\nu^*(D) + \epsilon_k,\epsilon_k/256)}{\epsilon_k}  (1 + \frac{\nu^*(D)}{\epsilon_k}) 
\end{eqnarray*}
Here the last line follows from Equation~\eqref{eqn:ubphiknonrealizable}. Hence the total number of examples queried is at most:
\[ \sum_{k=1}^{\lceil \log\frac{1}{\epsilon} \rceil} m_k \leq c_4 \sum_{k=1}^{\lceil \log\frac{1}{\epsilon} \rceil} ( d\ln \frac{\phi(2\nu^*(D) + \epsilon_k,\epsilon_k/256)}{\epsilon_k}  + \ln(\frac{k_0 - k + 1}{\delta}) ) \frac{\phi(2\nu^*(D) + \epsilon_k,\epsilon_k/256)}{\epsilon_k}  (1 + \frac{\nu^*(D)}{\epsilon_k}) \]
\end{proof}

\begin{proof}(of Theorem~\ref{thm:labelcomplexitytnc})
Assume $E_a$ happens.\\
First, from Equation~\eqref{eqn:unluncoverage} of Lemma~\ref{lem:unldatalp} when running Algorithm~\ref{alg:optcrp},
\begin{equation}\label{eqn:ubphiktnc}
 \phi_k \leq \bm{\Phi}_D(V_k, \frac{\epsilon_k}{128}) + \frac{\epsilon_k}{256} \leq \bm{\Phi}_D(B_D(h^*,C_0 \epsilon_k^{\frac{1}{\kappa}}), \frac{\epsilon_k}{128}) + \frac{\epsilon_k}{256} \leq \bm{\Phi}_D(B_D(h^*,C_0 \epsilon_k^{\frac{1}{\kappa}}), \frac{\epsilon_k}{256}) = \phi(C_0 \epsilon_k^{\frac{1}{\kappa}}, \frac{\epsilon_k}{256})
\end{equation}
where the second inequality follows from the fact that $V_k \subseteq B_D(h^*(D), C_0 \epsilon_k^{\frac{1}{\kappa}})$, and the third inequality follows from Lemma~\ref{lem:uncoveragewrterrguar} and denseness assumption.\\
Second, for all $h \in V_k$,
\begin{eqnarray*}
&&\phi_k \rho_{\tilde{\Gamma}_k}(h,h^*(D)) \\
&=&\E_{\tilde{U}_k}I(h(x) \neq h^*(D)(x)) \gamma_k(x) \\
&\leq& \rho_{\tilde{U}_k}(h,h^*(D)) \\
&\leq& \rho_D(h,h^*(D)) + \epsilon_k /32 \\
&\leq& C_0(\err_D(h) - \err_D(h^*(D)))^{\frac{1}{\kappa}} + \epsilon_k /32 \\
&\leq& C_0(\err_{\tilde{U}_k}(h) - \err_{\tilde{U}_k}(h^*(D)) + \epsilon_k/64)^{\frac{1}{\kappa}} + \epsilon_k /32 \\
&=& C_0(\E_{\tilde{U}_k} [I(h(x) \neq y)\gamma_k(x)] - \E_{\tilde{U}_k} [I(h^*(D)(x) \neq y)\gamma_k(x)] \\ 
&& + \E_{\tilde{U}_k} [I(h(x) \neq y) (1-\gamma_k(x))] - \E_{\tilde{U}_k} [I(h^*(D)(x) \neq y) (1 - \gamma_k(x))] + \epsilon_k/16)^{\frac{1}{\kappa}} + \epsilon_k /32 \\
\end{eqnarray*}
Here the first inequality follows from $\gamma_k(x) \leq 1$, the second inequality follows from Equation~\eqref{eqn:unldist} of Lemma~\ref{lem:unldata}, the third inequality follows from Definition~\ref{def:tnc} and the fourth inequality follows from Equation~\eqref{eqn:unlerrdiff} of Lemma~\ref{lem:unldata}. The above can be upper bounded by:

\begin{eqnarray*}
&\leq& C_0(\E_{\tilde{U}_k} [I(h(x) \neq y)\gamma_k(x)] - \E_{\tilde{U}_k} [I(h^*(D)(x) \neq y)\gamma_k(x)] + \epsilon_k/16)^{\frac{1}{\kappa}} + \epsilon_k /32 \\
&\leq& 2C_0(\E_{\tilde{U}_k} [I(h(x) \neq y)\gamma_k(x)] - \E_{\tilde{U}_k} [I(h^*(D)(x) \neq y)\gamma_k(x)] )^{\frac{1}{\kappa}} + 2C_0(\epsilon_k/16)^{\frac{1}{\kappa}} + \epsilon_k /32 \\
&\leq& \max(8C_0,4) \max((\E_{\tilde{U}_k}[I(h(x) \neq y)\gamma_k(x)] - \E_{\tilde{U}_k}[I(h^*(D)(x) \neq y)\gamma_k(x)]),\frac{\epsilon_k}{16})^{\frac{1}{\kappa}} \\
&=& \max(8C_0,4) (\phi_k)^{\frac{1}{\kappa}} \max(\P_{\tilde{\Gamma}_k}(h(x) \neq y) - \P_{\tilde{\Gamma}_k}(h^*(D)(x) \neq y),\frac{\epsilon_k}{8\phi_k})^{\frac{1}{\kappa}}
\end{eqnarray*}

Here the first inequality follows from Equation~\eqref{eqn:unldisagree} of Lemma~\ref{lem:unldatalp} and triangle inequality $\E_{\tilde{U}_k}[I(h(x) \neq y)\gamma_k(x)] - \E_{\tilde{U}_k}[I(h^*(D)(x) \neq y)\gamma_k(x)] \leq \E_{\tilde{U}_k}[I(h(x) \neq h^*(D)(x))\gamma_k(x)] \leq \epsilon_k/32$, and the last two inequalities follow from simple algebra.

Dividing both sides by $\phi_k$, we get:
\[ \rho_{\tilde{\Gamma}_k}(h,h^*(D)) \leq C_1 (\phi_k)^{\frac{1}{\kappa}-1} \max(\err_{\tilde{\Gamma}_k}(h) - \err_{\tilde{\Gamma}_k}(h^*(D)), \frac{\epsilon_k}{8\phi_k})^{\frac{1}{\kappa}} \]
where $C_1 = \max(8C_0,4)$.
Thus in iteration $k$, Condition~\eqref{eqn:approxtnc} in Lemma~\ref{lem:adaptivetotnc} holds with $C: = C_1 (\phi_k)^{\frac{1}{\kappa}-1}$ and $\tilde{h} := h^*(D)$. Thus, from Lemma~\ref{lem:adaptivetotnc}, Algorithm~\ref{alg:adaptive} succeeds,  and there exists a constant $c_5>0$, such that the number of labels queried is
\begin{eqnarray*}
 m_k &\leq& c_2 \max( (d\ln(C_1(\phi_k)^{\frac{1}{\kappa}-1}(\frac{\epsilon_k}{8\phi_k})^{\frac{1}{\kappa}-2}) + \ln\frac{2}{\delta_k}) (C_1(\phi_k)^{\frac{1}{\kappa}-1} (\frac{\epsilon_k}{8\phi_k})^{\frac{1}{\kappa}-2} ),\\ 
 && (d\ln(\frac{\epsilon_k}{8\phi_k})^{-1} + \ln{\frac{2}{\delta_k}} ) (\frac{\epsilon_k}{8\phi_k})^{-1} ) \\
 &\leq& c_5(d\ln(\phi_k \epsilon_k^{\frac{1}{\kappa}-2}) + \ln(\frac{k_0 - k + 1}{\delta})) \phi_k \epsilon_k^{\frac{1}{\kappa}-2} \\
 &\leq& c_5(d\ln (\phi(C_0\epsilon_k^{\frac{1}{\kappa}},\frac{\epsilon_k}{256}) \epsilon_k^{\frac{1}{\kappa}-2}) + \ln(\frac{k_0 - k + 1}{\delta})) \phi(C_0\epsilon_k^{\frac{1}{\kappa}},\frac{\epsilon_k}{256}) \epsilon_k^{\frac{1}{\kappa}-2}
\end{eqnarray*}
Where the last line follows from Equation~\eqref{eqn:ubphiknonrealizable}. Hence the total number of examples queried is at most
\[ \sum_{k=1}^{\lceil \log\frac{1}{\epsilon} \rceil} m_k \leq c_5\sum_{k=1}^{\lceil \log\frac{1}{\epsilon} \rceil} (d\ln (\phi(C_0\epsilon_k^{\frac{1}{\kappa}},\frac{\epsilon_k}{256}) \epsilon_k^{\frac{1}{\kappa}-2}) + \ln(\frac{k_0 - k + 1}{\delta})) \phi(C_0\epsilon_k^{\frac{1}{\kappa}},\frac{\epsilon_k}{256}) \epsilon_k^{\frac{1}{\kappa}-2} \]
\end{proof}

The following lemma is an immediate corollary of Theorem 21, item (a) of Lemma 2 and Lemma 3 of~\cite{BL13}:
\begin{lemma}\label{lem:logconcavephi}
Suppose $D$ is isotropic and log-concave on $\R^d$, and $\calH$ is the set of homogeneous linear classifiers on $\R^d$, then there exist absolute constants $c_6,c_7 > 0$ such that $\phi(r,\eta) \leq c_6 r \ln \frac{c_7 r}{\eta}$.
\end{lemma}

\begin{proof}(of Lemma~\ref{lem:logconcavephi})
Denote $w_h$ as the unit vector $w$ such that $h(x) = \sign(w \cdot x)$, and $\theta(w,w')$ to be the angle between vectors $w$ and $w'$.
If $h \in B_D(h^*,r)$, then by Lemma 3 of~\cite{BL13}, there exists some constant $c_{11} > 0$ such that $\theta(w_h,w_{h^*}) \leq \frac{r}{c_{11}}$. Also, by Lemma 21 of~\cite{BL13}, there exists some constants $c_{12}, c_{13} > 0$, such that, if $\theta(w,w') = \alpha$ then 
\[ \P_D( \sign(w \cdot x) \neq \sign(w' \cdot x) ,|w \cdot x | \geq b) \leq c_{12} \alpha \exp(-c_{13} \frac{b}{\alpha})  \]
We define a special solution $(\xi, \zeta, \gamma)$ as follows:
\[ \xi(x) := I( w_{h^*} \cdot x \geq \frac{r}{c_{11} c_{13}}\ln\frac{c_{12} r}{c_{11}\eta} ) \]
\[ \zeta(x) := I( w_{h^*} \cdot x \leq -\frac{r}{c_{11} c_{13}}\ln\frac{c_{12} r}{c_{11}\eta} ) \]
\[ \gamma(x) := I( |w_{h^*} \cdot x| \leq \frac{r}{c_{11} c_{13}}\ln\frac{c_{12} r}{c_{11}\eta} ) \]
Then it can be checked that for all $h \in B_D(h^*,r)$,
\[ \E [I(h(x) = +1) \zeta(x) + I(h(x) = -1) \xi(x)] = \P_D( \sign(w_{h^*} \cdot x) \neq \sign(w_h \cdot x) ,|w_{h^*} \cdot x | \geq \frac{r}{c_{11} c_{13}}\ln\frac{c_{12} r}{c_{11}\eta}) \leq \eta \]
And by item (a) of Lemma 2 of~\cite{BL13}, we have
 \[ \E \gamma(x) = \P_D( |w_{h^*} \cdot x| \leq \frac{r}{c_{11} c_{13}}\ln\frac{c_{12} r}{c_{11}\eta}) \leq \frac{r}{c_{11} c_{13}}\ln\frac{c_{12} r}{c_{11}\eta} \]
Hence, 
\[ \phi(r,\eta) \leq \frac{r}{c_{11} c_{13}}\ln\frac{c_{12} r}{c_{11}\eta}\]
\end{proof}

\begin{proof}(of Corollary~\ref{cor:logconcave})
This is an immediate consequence of Lemma~\ref{lem:logconcavephi} and Theorems~\ref{thm:labelcomplexity} and~\ref{thm:labelcomplexitytnc} and algebra.
\end{proof}

\section{A Suboptimal Alternative to Algorithm~\ref{alg:adaptive} }
\begin{algorithm}[H]
\caption{An Nonadaptive Algorithm for Label Query Given Target Excess Error}
\label{alg:nonadaptive}
\begin{algorithmic}[1]
\State {\bf{Inputs:}} Hypothesis set $V$ of VC dimension $d$, Example distribution $\Delta$, Labeling oracle $\calO$, target excess error $\tilde{\epsilon}$, target confidence $\tilde{\delta}$.

\State Draw $n = \frac{6144}{\tilde{\epsilon}^2} (d\ln\frac{6144}{\tilde{\epsilon}^2} + \ln\frac{24}{\tilde{\delta}})$ i.i.d examples from $\Delta$; query their labels from $\calO$ to get a labelled dataset $S$.
\State Train an ERM classifier $\hat{h} \in V$ over $S$.
\State Define the set $V$ as follows: 
\[ V_1 = \Big{\{}h \in V: \err_S(h) \leq \err_S(\hat{h}) +  \frac{3\tilde{\epsilon}}{4} \Big{\}} \]
\State \Return $V_1$.
\end{algorithmic}
\end{algorithm}
It is immediate that we have the following lemma.
\begin{lemma}
\label{lem:ratiotypenonadpative}
Suppose we run Algorithm~\ref{alg:nonadaptive} with inputs hypothesis set $V$, example distribution $\Delta$, labelling oracle $\calO$, target excess error $\tilde{\epsilon}$ and target confidence $\tilde{\delta}$. Then there exists an event $\tilde{E}$, $\P(\tilde{E}) \geq 1 - \tilde{\delta}$, such that on $\tilde{E}$, the set $V_1$ has the following property. (1) If for $h \in \calH$, $\err_{\tilde{\Delta}}(h) - \err_{\tilde{\Delta}}(h^*(\tilde{\Delta})) \leq \tilde{\epsilon}/2$, then $h \in V_1$. (2) On the other hand, if $h \in V_1$, then $\err_{\tilde{\Delta}}(h) - \err_{\tilde{\Delta}}(h^*(\tilde{\Delta})) \leq \tilde{\epsilon}$.
\end{lemma}
When $\tilde{E}$ happens, we say that Algorithm~\ref{alg:nonadaptive} succeeds.
\begin{proof}
By Equation~\eqref{eqn:adderrdiff} of Lemma~\ref{lem:addvc} and because $n = \frac{6144}{\tilde{\epsilon}^2} (d\ln\frac{6144}{\tilde{\epsilon}^2} + \ln\frac{24}{\tilde{\delta}})$, we have for all $h,h' \in \calH$,
\[ (\err_{\tilde{\Delta}}(h) - \err_{\tilde{\Delta}}(h')) - (\err_S(h) - \err_S(h')) \leq \frac{\epsilon}{4}\]
For the proof of (1), for any $h \in V$, $\err_{\tilde{\Delta}}(h) - \err_{\tilde{\Delta}}(h^*(\tilde{\Delta})) \leq \tilde{\epsilon}/2$, then
\[ \err_{\tilde{\Delta}}(h) - \err_{\tilde{\Delta}}(\hat{h}) \leq \tilde{\epsilon}/2 \]
Thus
\[ \err_S(h) - \err_S(\hat{h}) \leq \frac{3\tilde{\epsilon}}{4} \]
proving $h \in V_1$.\\
For the proof of (2), for any $h \in V_1$, 
\[ \err_S(h) - \err_S(h') \leq \frac{3\tilde{\epsilon}}{4} \]
Thus
\[ \err_S(h) - \err_S(h^*(\tilde{\Delta})) \leq \frac{3\tilde{\epsilon}}{4} \]
Combining with the fact that $ (\err_{\tilde{\Delta}}(h) - \err_{\tilde{\Delta}}(h^*(\tilde{\Delta}))) - (\err_S(h) - \err_S(h^*(\tilde{\Delta}))) \leq \frac{\epsilon}{4} $
we have
\[ \err_{\tilde{\Delta}}(h) - \err_{\tilde{\Delta}}(h^*(\tilde{\Delta})) \leq \tilde{\epsilon} \]
\end{proof}

\begin{corollary}
\label{cor:labelcomplexityaddvc}
Suppose we replace the calls to Algorithm~\ref{alg:adaptive} with Algorithm~\ref{alg:nonadaptive} in Algorithm~\ref{alg:labelquery}, then run it with inputs example oracle $\calU$, labelling oracle $\calO$, hypothesis class $V$, confidence-rated predictor $P$ of Algorithm~\ref{alg:optcrp}, target excess error $\epsilon$ and target confidence $\delta$. Then the modified algorithm has a label complexity of 
\[ \tilde{O}(\sum_{k=1}^{\lceil \log1/\epsilon \rceil} (d (\frac{\phi(2\nu^*(D) + \epsilon_k,\epsilon_k/256)}{\epsilon_k})^2 )\]
in the agnostic case and 
\[ \tilde{O}(\sum_{k=1}^{\lceil \log1/\epsilon \rceil} d(\frac{\phi(C_0\epsilon_k^{\frac{1}{\kappa}},\frac{\epsilon_k}{256})}{\epsilon_k^{\frac{1}{\kappa}}} )^2 \epsilon_k^{\frac{2}{\kappa}-2})\]
under $(C_0,\kappa)$-Tsybakov Noise Condition.
\end{corollary}

Under denseness assumption, by Lemma~\ref{lem:uncoveragetriviallb}, we have $\phi(r,\eta) \geq r - 2\eta$, the label complexity bounds given by Corollary~\ref{cor:labelcomplexityaddvc} is always no better than the ones given by Theorem~\ref{thm:labelcomplexity} and~\ref{thm:labelcomplexitytnc}. 

\begin{proof}(Sketch)
Define event
\begin{eqnarray*}
&&\text{
$E_a =$ \{For all $k = 1,2,\ldots,k_0$: Equations~\eqref{eqn:unlerr},~\eqref{eqn:unlerrdiff},~\eqref{eqn:unldist},~\eqref{eqn:unldisagree},~\eqref{eqn:unluncoverage} hold for $\tilde{U}_k$ with }\\
&&\text{
confidence $\delta_k/2$, and Algorithm~\ref{alg:nonadaptive} succeeds with inputs hypothesis set $V = V_k$, example } \\
&&\text{
distribution $\Delta = \Gamma_k$, labelling oracle $\calO$, target excess error $\tilde{\epsilon} = \frac{\epsilon_k}{8\phi_k}$ and target confidence $\tilde{\delta} = \frac{\delta_k}{2}$\}. }
\end{eqnarray*}
Clealy, $\P(E_a) \geq 1 - \delta$. On the event $E_a$, there exists an absolute constant $c_{13} > 0$, such that the number of examples queried in interation $k$ is
\[ m_k \leq c_{13}(d\ln\frac{8\phi_k}{\epsilon_k} + \ln\frac{2}{\delta}) (\frac{8\phi_k}{\epsilon_k})^2 \]
Combining it with Equation~\eqref{eqn:unluncoverage} of Lemma~\ref{lem:unldatalp}
\[ \phi_k \leq \bm{\Phi}_D(V_k, \frac{\epsilon_k}{128}) + \frac{\epsilon_k}{256} \]
we have
\[ m_k \leq O ((d\ln\frac{\bm{\Phi}_D(V_k, \frac{\epsilon_k}{128}) + \frac{\epsilon_k}{256}}{\epsilon_k} + \ln\frac{2}{\delta_k}) (\frac{\bm{\Phi}_D(V_k, \frac{\epsilon_k}{128}) + \frac{\epsilon_k}{256}}{\epsilon_k})^2) \]
The rest of the proof follows from Lemma~\ref{lem:uncoveragewrterrguar} and denseness assumption, along with algebra.
\end{proof}

\section{Proofs of Concentration Lemmas}

\begin{proof}(of Lemma~\ref{lem:unldata})
We begin by observing that:
\[ \err_{\tilde{U}_k}(h) = \frac{1}{n_k} \sum_{i=1}^{n_k} [\P_D(Y=+1|X=x_i)I(h(x_i) = -1) + \P_D(Y=-1|X=x_i) I(h(x_i) = +1)] \]
Moreover, $\max(\calS(\{I(h(x) = 1, h \in \calH)\}, n),\calS(\{I(h(x) = -1, h \in \calH)\}, n)) \leq (\frac{en}{d})^d$. Combining this fact with Lemma~\ref{lem:wtaddvc}, the following equations hold simultaneously with probability $1-\delta_k/6$:
\[ \Big{|} \frac{1}{n_k} \sum_{i=1}^{n_k} \P_D(Y=+1|X=x_i)I(h(x_i) = -1) - \P_D(h(x) = -1, y = +1) \Big{|} \leq \sqrt{\frac{8(d\ln \frac{en_k}{d} + \ln \frac{24}{\delta_k})}{n_k}} \leq \frac{\epsilon_k}{128} \]
\[ \Big{|} \frac{1}{n_k} \sum_{i=1}^{n_k} \P_D(Y=-1|X=x_i)I(h(x_i) = +1) - \P_D(h(x) = +1, y = -1) \Big{|} \leq \sqrt{\frac{8(d\ln \frac{en_k}{d} + \ln \frac{24}{\delta_k})}{n_k}} \leq \frac{\epsilon_k}{128} \]
Thus Equation~\eqref{eqn:unlerr} holds with probability $1 - \delta_k/6$. Moreover, we observe that Equation~\eqref{eqn:unlerr} implies Equation~\eqref{eqn:unlerrdiff}.
To show Equation~\eqref{eqn:unldist}, we observe that by Lemma~\ref{lem:addvc}, with probability $1 - \delta_k/12$,
\[ | \rho_D(h,h') - \rho_{\tilde{U}_k}(h,h') | = | \rho_D(h,h') - \rho_{S_k}(h,h') | \leq 2\sqrt{\sigma(n_k, \delta_k/12)} \leq \frac{\epsilon_k}{64} \]
Thus, Equation~\eqref{eqn:unldist} holds with probability $\geq 1 - \delta_k/12$. By union bound, with probability $1-\delta_k/4$, Equations~\eqref{eqn:unlerr}, ~\eqref{eqn:unlerrdiff}, and~\eqref{eqn:unldist} hold simultaneously.
\end{proof}

\begin{proof}(of Lemma~\ref{lem:unldatalp})
(1) Given a confidence-rated predictor with inputs hypothesis set $V_k$, unlabelled data $U_k$, and error bound $\epsilon_k/64$, the outputs $\{(\xi_{k,i}, \zeta_{k,i}, \gamma_{k,i})\}_{i=1}^{n_k}$ must satisfy that for all $h,h' \in V_k$,
\[ \frac{1}{n_k} \sum_{i=1}^{n_k} [I(h(x_{k,i})=-1) \xi_{k,i} + I(h(x_{k,i})=+1) \zeta_{k,i}] \leq \frac{\epsilon_k}{64} \]
\[ \frac{1}{n_k} \sum_{i=1}^{n_k} [I(h'(x_{k,i})=-1) \xi_{k,i} + I(h'(x_{k,i})=+1) \zeta_{k,i}] \leq \frac{\epsilon_k}{64} \]
Since $I(h(x) \neq h'(x)) \leq \min(I(h(x)=-1) + I(h'(x)=-1), I(h(x)=+1) + I(h'(x)=+1))$, adding up the two inequalities above, we get
\[ \frac{1}{n_k} \sum_{i=1}^{n_k} [I(h(x_{k,i}) \neq h'(x_{k,i}))(\xi_{k,i} + \zeta_{k,i})] \leq \frac{\epsilon_k}{32} \]
That is,
\[ \frac{1}{n_k} \sum_{i=1}^{n_k} [I(h(x_{k,i}) \neq h'(x_{k,i}))(1 - \gamma_{k,i})] \leq \frac{\epsilon_k}{32} \]

(2) By definition of $\bm{\Phi}_D(V,\eta)$, there exist nonnegative functions $\xi, \zeta, \gamma$ such that $\xi(x) + \zeta(x) + \gamma(x) \equiv 1$, $\E_D [\gamma(x)] = \bm{\Phi}_D(V_k,\epsilon_k/128)$ and for all $h \in V_k$,
 \[ \E_D [\xi(x) I(h(x) = -1) + \zeta(x) I(h(x) = +1)] \leq \frac{\epsilon_k}{128} \]

Consider the linear progam in Algorithm~\ref{alg:optcrp} with inputs hypothesis set $V_k$, unlabelled data $U_k$, and error bound $\epsilon_k/64$.  We consider the following special (but possibly non-optimal) solution for this LP: $\xi_{k,i} = \xi(z_{k,i}), \zeta_{k,i} = \zeta(z_{k,i}), \gamma_{k, i} = \gamma(z_{k, i})$. We will now show that this solution is feasible and has coverage $\bm{\Phi}_D(V_k,\epsilon_k/128)$ plus $O(\epsilon_k)$ with high probability.\\
Observe that $\max(\calS(\{I(h(x) = 1, h \in \calH)\}, n),\calS(\{I(h(x) = -1, h \in \calH)\}, n)) \leq (\frac{en}{d})^d$.
Therefore, from Lemma~\ref{lem:wtaddvc} and the union bound, with probability $1 - \delta_k / 4$, the following hold simultaneously for all $h \in \calH$:
\begin{equation}\label{eqn:gammaconc}
 \Big{|} \frac{1}{n_k} \sum_{i=1}^{n_k} \gamma(z_{k,i}) - \E_D \gamma(x) \Big{|} \leq \sqrt{\frac{\ln\frac{2}{\delta_k}}{2n_k}} \leq \frac{\epsilon_k}{256}  
\end{equation}
\begin{equation}\label{eqn:xiconc}
 \Big{|} \frac{1}{n_k} \sum_{i=1}^{n_k} \xi(z_{k,i})I(h(z_{k,i}) = -1) - \E_D [\xi(x) I(h(x) = -1)] \Big{|} \leq \sqrt{\frac{8(d\ln \frac{en_k}{d} + \ln \frac{24}{\delta_k})}{n_k}} \leq \frac{\epsilon_k}{256}
\end{equation}
\begin{equation}\label{eqn:zetaconc}
 \Big{|} \frac{1}{n_k} \sum_{i=1}^{n_k} \zeta(z_{k,i})I(h(z_{k,i}) = +1) - \E_D [\zeta(x) I(h(x) = +1)] \Big{|} \leq \sqrt{\frac{8(d\ln \frac{en_k}{d} + \ln \frac{24}{\delta_k})}{n_k}} \leq \frac{\epsilon_k}{256}
\end{equation}
Adding up Equations~\eqref{eqn:xiconc} and~\eqref{eqn:zetaconc},
\[ \Big{|} \frac{1}{n_k} \sum_{i=1}^{n_k} [\zeta(x_i)I(h(x_i) = +1) + \xi(x_i)I(h(x_i) = -1)] - \E_D [\xi(x) I(h(x) = -1) + \zeta(x) I(h(x) = +1)) ] \Big{|} \leq \frac{\epsilon_k}{128} \]
Thus $\{(\xi(z_{k,i}), \zeta(z_{k,i})\}_{i=1}^{n_k}$ is a feasible solution of the linear program of Algorithm~\ref{alg:optcrp}. Also, by Equation~\eqref{eqn:gammaconc}, $\frac{1}{n_k} \sum_{i=1}^{n_k} \gamma(z_{k,i}) \leq \bm{\Phi}_D(V_k,\frac{\epsilon_k}{128}) + \frac{\epsilon_k}{64}$. Thus, the outputs $\{(\xi_{k,i}, \zeta_{k,i}, \gamma_{k,i})\}_{i=1}^{n_k}$ of the linear program in Algorithm~\ref{alg:optcrp} satisfy 
\[\phi_k = \frac{1}{n_k} \sum_{i=1}^{n_k} \gamma_{k,i} \leq \frac{1}{n_k} \sum_{i=1}^{n_k} \gamma(z_{k,i}) \leq \bm{\Phi}_D(V_k,\frac{\epsilon_k}{128}) + \frac{\epsilon_k}{256} \] 
due to their optimality.
\end{proof}

\begin{lemma}
\label{lem:wtaddvc}
Pick any $n \geq 1$, $\delta \in (0,1)$, a family $\calF$ of functions $f: \calZ \to \{0,1\}$, a fixed weighting function $w: \calZ \to [0,1]$. Let $S_n$ be a set of $n$ iid copies of $Z$. The following holds with probability at least $1 - \delta$:
\[  \Big{|}\frac{1}{n} \sum_{i=1}^n w(z_i) f(z_i) - \E [w(z)f(z)] \Big{|} \leq \sqrt{\frac{ 8(\ln\calS(\calF,n) + \ln\frac{2}{\delta} ) }{n} } \]
where $\calS(\calF,n) = \max_{z_1,\ldots,z_n \in \calZ} |\{(f(z_1),\ldots,f(z_n)): f \in \calF\}|$ is the growth function of $\calF$.
\end{lemma}

\begin{proof}
The proof is fairly standard, and follows immediately from the proof of additive VC bounds. With probability $1 - \delta$,
\begin{eqnarray*}
&&\sup_{f \in \calF} \Big{|} \frac{1}{n} \sum_{i=1}^{n} w(z_i) f(z_i) - \E w(z) f(z) \Big{|} \\
&\leq& \E_{S \sim D^n} \sup_{f \in \calF} \Big{|} \frac{1}{n} \sum_{i=1}^{n} w(z_i) f(z_i) - \E w(z) f(z) \Big{|} + \sqrt{\frac{2\ln\frac{1}{\delta}}{n}} \\
&\leq& \E_{S \sim D^n, S' \sim D^n} \sup_{f \in \calF} \Big{|} \frac{1}{n} \sum_{i=1}^{n} (w(z_i) f(z_i) -  w(z'_i) f(z'_i)) \Big{|} + \sqrt{\frac{2\ln\frac{1}{\delta}}{n}} \\
&\leq& \E_{S \sim D^n, S' \sim D^n, \sigma \sim U(\{-1,+1\}^n)} \sup_{f \in \calF} \Big{|} \frac{1}{n} \sum_{i=1}^{n} \sigma_i(w(z_i) f(z_i) - w(z'_i) f(z'_i)) \Big{|} + \sqrt{\frac{2\ln\frac{1}{\delta}}{n}} \\
&\leq& 2\E_{S \sim D^n,\sigma \sim U(\{-1,+1\}^n)} \sup_{f \in \calF} \Big{|}\frac{1}{n} \sum_{i=1}^{n} \sigma_i w(z_i) f(z_i) \Big{|} + \sqrt{\frac{2\ln\frac{1}{\delta}}{n}}\\
&\leq& 2\sqrt{\frac{2\ln(2\calS(\calF,n))}{n}} + \sqrt{\frac{2\ln\frac{1}{\delta}}{n}} \leq \sqrt{\frac{ 8(\ln\calS(\calF,n) + \ln\frac{2}{\delta} ) }{n} } \\
\end{eqnarray*}
Where the first inequality is by McDiarmid's Lemma; the second inequality follows from Jensen's Inequality; the third inequality follows from symmetry; the fourth inequality follows from $|A+B| \leq |A|+|B|$; the fifth inequality follows from Massart's Finite Lemma. 
\end{proof}

\begin{lemma}
\label{lem:uncoveragetriviallb}
Let $0 < 2 \eta \leq r \leq 1$. Given a hypothesis set $V$ and data distribution $D$ over $\calX \times \calY$, if there exist $h_1, h_2 \in V$ such that $\rho_D(h_1,h_2) \geq r$, then $\bm{\Phi}_D(V,\eta) \geq r-2\eta$.
\end{lemma}
\begin{proof}
Let $(\xi,\zeta,\gamma)$ be a triple of functions from $\calX$ to $\R^3$ satisfying the following conditions: $\xi,\zeta,\gamma \geq 0$, $\xi+\zeta+\gamma \equiv 1$, and for all $h \in V$,
\[ \E_D [\xi(x)I(h(x)=+1) + \zeta(x)I(h(x)=-1)] \leq \eta \]
Then, in particular, we have:
\[ \E_D [\xi(x)I(h_1(x)=+1) + \zeta(x)I(h_1(x)=-1)] \leq \eta \]
\[ \E_D [\xi(x)I(h_1(x)=+1) + \zeta(x)I(h_2(x)=-1)] \leq \eta \]
Thus, by $I(h_1(x) \neq h_2(x)) \leq \min(I(h_1(x)=-1) + I(h_1(x)=-1), I(h_2(x)=+1) + I(h_2(x)=+1))$, adding the two inequalities up,
\[ \E_D [(\xi(x) + \zeta(x)) I(h_1(x) \neq h_2(x))] \leq 2\eta \]
Since 
\[ \rho_D(h_1,h_2) = \E_D I(h_1(x) \neq h_2(x)) \geq r \]
We have
\[ \E_D [\gamma(x) I(h_1(x) \neq h_2(x))] = \E_D [(1 - \xi(x) - \zeta(x)) I(h_1(x) \neq h_2(x))] \geq r - 2\eta \]
Thus,
\[ \E_D [\gamma(x)] \geq \E_D [\gamma(x) I(h_1(x) \neq h_2(x))] \geq r - 2\eta\]
Hence $\bm{\Phi}_D(V,\eta) \geq r-2\eta$.
\end{proof}

\begin{lemma}
\label{lem:uncoveragewrterrguar}
Given hypothesis set $V$ and data distribution $D$ over $\calX \times \calY$, $0<\lambda<\eta<1$, if there exist $h_1, h_2 \in V$ such that $\rho_D(h_1, h_2) \geq 2\eta - \lambda$, then $\bm{\Phi}_D(V,\eta) + \lambda \leq \bm{\Phi}_D(V,\eta-\lambda)$. 
\end{lemma}
\begin{proof}
Suppose $(\xi_1, \zeta_1, \gamma_1)$ are nonnegative functions satisfying $\xi_1 + \zeta_1 + \gamma_1 \equiv 1$, and for all $h \in V$, $\E_D[\zeta_1(x)I(h(x) = +1) + \xi_1(x)I(h(x) = -1)] \leq \eta - \lambda$, and $\E_D \gamma_1(x) = \bm{\Phi}_D(V,\eta - \lambda)$. Notice by Lemma~\ref{lem:uncoveragetriviallb},$ \bm{\Phi}_D(V,\eta - \lambda) \geq 2\eta - \lambda - 2 (\eta - \lambda) = \lambda$.

Then we pick nonnegative functions $(\xi_2, \zeta_2, \gamma_2)$ as follows. Let $\xi_2 = \xi_1$, $\gamma_2 = (1 - \frac{\lambda}{\bm{\Phi}_D(V,\eta - \lambda)}) \gamma_1$, and $\zeta_2 = 1 - \xi_2 - \gamma_2$. It is immediate that $(\xi_2, \zeta_2, \gamma_2)$ is a valid confidence rated predictor and $\zeta_2 \geq \zeta_1$, $\gamma_2 \leq \gamma_1$, $\E_D\gamma_2(x) = \bm{\Phi}_D(V,\eta - \lambda) - \lambda$. It can be readily checked that the confidence rated predictor $(\xi_2, \zeta_2, \gamma_2)$ has error guarantee $\eta$, specifically:
\begin{eqnarray*}
&& \E_D [\zeta_2(x)I(h(x) = +1) + \xi_2(x)I(h(x) = -1)] \\
&\leq& \E_D [(\zeta_2(x) - \zeta_1(x)) I(h(x) = +1) + (\xi_2(x) - \xi_1(x)) I(h(x) = -1)] + \eta - \lambda \\
&\leq& \E_D [(\zeta_2(x) - \zeta_1(x)) + (\xi_2(x) - \xi_1(x))] + \eta - \lambda\\
&\leq& \lambda + \eta - \lambda = \eta
\end{eqnarray*}
Thus, $\bm{\Phi}_D(V,\eta)$, which is the minimum abstention probability of a confidence-rated predictor with error guarantee $\eta$ with respect to hypothesis set $V$ and data distribution $D$, is at most $\bm{\Phi}_D(V,\eta-\lambda) - \lambda$.
\end{proof}

%% file: derivation.tex
\section{Detailed Derivation of Label Complexity Bounds}
\subsection{Agnostic}
\begin{proposition}
\label{prop:agnosticlc}
In agnostic case, the label complexity of Algorithm~\ref{alg:labelquery} is at most
\[ \tilde{O}( \sup_{k \leq \lceil \log(1/\epsilon) \rceil} \frac{\phi(2 \nu^*(D) + \epsilon_k, \epsilon_k/256)}{2 \nu^*(D) + \epsilon_k} (d \frac{\nu^*(D)^2}{\epsilon^2}\ln\frac{1}{\epsilon} + d \ln^2\frac{1}{\epsilon}) ), \]
where the $\tilde{O}$ notation hides factors logarithmic in $1/\delta$.
\end{proposition}
\begin{proof}
Applying Theorem~\ref{thm:labelcomplexitytnc}, the total number of labels queried is at most: 
\[ c_4 \sum_{k=1}^{\lceil \log\frac{1}{\epsilon} \rceil} ( d\ln \frac{\phi(2\nu^*(D) + \epsilon_k,\epsilon_k/256)}{\epsilon_k}  + \ln(  \frac{ \lceil\log(1/\epsilon) \rceil - k + 1}{\delta}) ) \frac{\phi(2\nu^*(D) + \epsilon_k,\epsilon_k/256)}{\epsilon_k}  (1 + \frac{\nu^*(D)}{\epsilon_k})  \]
Using the fact that $\phi(2\nu^*(D) + \epsilon_k,\epsilon_k/256) \leq 1$, this is
\begin{eqnarray*} 
&&c_4 \sum_{k=1}^{\lceil \log\frac{1}{\epsilon} \rceil} ( d\ln \frac{\phi(2\nu^*(D) + \epsilon_k,\epsilon_k/256)}{\epsilon_k}  + \ln(  \frac{ \lceil\log(1/\epsilon) \rceil - k + 1}{\delta}) ) \frac{\phi(2\nu^*(D) + \epsilon_k,\epsilon_k/256)}{\epsilon_k}  (1 + \frac{\nu^*(D)}{\epsilon_k})  \\
&=& \tilde{O}\left( \sum_{k=1}^{\lceil \log\frac{1}{\epsilon} \rceil} ( d\ln \frac{\phi(2\nu^*(D) + \epsilon_k,\epsilon_k/256)}{\epsilon_k}  + \ln \log(1/\epsilon) ) \frac{\phi(2\nu^*(D) + \epsilon_k,\epsilon_k/256)}{2\nu + \epsilon_k}  (1 + \frac{\nu^*(D)^2}{\epsilon^2_k}) \right)  \\
&\leq& \tilde{O}\left( \sup_{k \leq \lceil \log(1/\epsilon) \rceil} \frac{\phi(2 \nu^*(D) + \epsilon_k, \epsilon_k/256)}{2 \nu^*(D) + \epsilon_k} \sum_{k=1}^{\lceil \log\frac{1}{\epsilon} \rceil} (1 + \frac{\nu^*(D)^2}{\epsilon_k^2}) (d \ln\frac{1}{\epsilon} +  \ln\ln\frac{1}{\epsilon} ) \right) \\
&\leq& \tilde{O}\left( \sup_{k \leq \lceil \log(1/\epsilon) \rceil} \frac{\phi(2 \nu^*(D) + \epsilon_k, \epsilon_k/256)}{2 \nu^*(D) + \epsilon_k} (d \frac{\nu^*(D)^2}{\epsilon^2}\ln\frac{1}{\epsilon} + d \ln^2\frac{1}{\epsilon}) \right),
\end{eqnarray*}
where the last line follows as $\epsilon_k$ is geometrically decreasing.
\end{proof}

\subsection{Tsybakov Noise Condition with $\kappa > 1$}
\begin{proposition}
\label{prop:tnclc}
Suppose the hypothesis class $\calH$ and the data distribution $D$ satisfies $(C_0,\kappa)$-Tsybakov Noise Condition with $\kappa > 1$. Then the label complexity of Algorithm~\ref{alg:labelquery} is at most
\[ \tilde{O}( \sup_{k \leq \lceil \log(1/\epsilon) \rceil} \frac{\phi(C_0\epsilon_k^{\frac{1}{\kappa}},\frac{\epsilon_k}{256})}{\epsilon_k^{\frac{1}{\kappa}}}  \epsilon^{\frac{2}{\kappa}-2} d\ln\frac{1}{\epsilon}  ), \]
where the $\tilde{O}$ notation hides factors logarithmic in $1/\delta$.
\end{proposition}
\begin{proof}
Applying Theorem~\ref{thm:labelcomplexitytnc}, the total number of labels queried is at most:
\[ c_5\sum_{k=1}^{\lceil \log\frac{1}{\epsilon} \rceil} (d\ln (\phi(C_0\epsilon_k^{\frac{1}{\kappa}},\frac{\epsilon_k}{256}) \epsilon_k^{\frac{1}{\kappa}-2}) + \ln(\frac{k_0 - k + 1}{\delta})) \phi(C_0\epsilon_k^{\frac{1}{\kappa}},\frac{\epsilon_k}{256}) \epsilon_k^{\frac{1}{\kappa}-2} \]
Using the fact that 
$\phi(C_0\epsilon_k^{\frac{1}{\kappa}},\frac{\epsilon_k}{256}) \leq 1$, we get
\begin{eqnarray*}
&&c_5\sum_{k=1}^{\lceil \log\frac{1}{\epsilon} \rceil} (d\ln (\phi(C_0\epsilon_k^{\frac{1}{\kappa}},\frac{\epsilon_k}{256}) \epsilon_k^{\frac{1}{\kappa}-2}) + \ln(\frac{k_0 - k + 1}{\delta})) \phi(C_0\epsilon_k^{\frac{1}{\kappa}},\frac{\epsilon_k}{256}) \epsilon_k^{\frac{1}{\kappa}-2} \\
&\leq& \tilde{O}\left(\sup_{k \leq \lceil \log(1/\epsilon) \rceil} \frac{\phi(C_0\epsilon_k^{\frac{1}{\kappa}},\frac{\epsilon_k}{256})}{\epsilon_k^{\frac{1}{\kappa}}}  \sum_{k=1}^{\lceil \log\frac{1}{\epsilon} \rceil} \epsilon_k^{\frac{2}{\kappa}-2} d\ln\frac{1}{\epsilon} \right) \\
&\leq& \tilde{O}\left( \sup_{k \leq \lceil \log(1/\epsilon) \rceil} \frac{\phi(C_0\epsilon_k^{\frac{1}{\kappa}},\frac{\epsilon_k}{256})}{\epsilon_k^{\frac{1}{\kappa}}}  \epsilon^{\frac{2}{\kappa}-2} d\ln\frac{1}{\epsilon} \right)
\end{eqnarray*}
\end{proof}

\subsection{Fully Agnostic, Linear Classification of Log-Concave Distribution}
We show in this subsection that in agnostic case, if $\calH$ is the class of homogeneous linear classifiers in $\R^d$, $D_{\calX}$ is isotropic log-concave in $\R^d$, then, our label complexity bound is at most 
\[ O( \ln\frac{\epsilon+\nu^*(D)}{\epsilon} (\ln\frac{1}{\epsilon} + \frac{\nu^*(D)^2}{\epsilon^2}) (d\ln\frac{\epsilon+\nu^*(D)}{\epsilon} + \ln\frac{1}{\delta}) + \ln\frac{1}{\epsilon} \ln\frac{\epsilon+\nu^*(D)}{\epsilon} \ln\ln\frac{1}{\epsilon} ) \]
Recall by Lemma~\ref{lem:logconcavephi}, we have $\phi(2 \nu^*(D) + \epsilon_k, \epsilon_k/256) \leq C ( \nu^*(D) + \epsilon_k ) \ln\frac{ \nu^*(D) + \epsilon_k}{ \epsilon_k } $ for some constant $C > 0$.
Applying Theorem~\ref{thm:labelcomplexity}, the label complexity is
\[  O( \sum_{k=1}^{\lceil \log\frac{1}{\epsilon} \rceil} ( d\ln(\frac{2\nu^*(D) + \epsilon_k}{\epsilon_k}\ln\frac{2\nu^*(D) + \epsilon_k}{\epsilon_k})  + \ln(\frac{\log(1/\epsilon) - k + 1}{\delta}) ) \ln\frac{2\nu^*(D) + \epsilon_k}{\epsilon_k}  (1 + \frac{\nu^*(D)^2}{\epsilon^2_k}) ) \]
This can be simplified to (treating $1$ and $\frac{\nu^*(D)^2}{\epsilon_k^2}$ separately)
\begin{eqnarray*}
&&O( \sum_{k=1}^{\lceil \log\frac{1}{\epsilon} \rceil} \ln\frac{\nu^*(D) + \epsilon_k}{\epsilon_k} (d \ln\frac{\nu^*(D) + \epsilon_k}{\epsilon_k} + \ln\frac{k_0 - k + 1}{\delta}) \\
&&+ \sum_{k=1}^{\lceil \log\frac{1}{\epsilon} \rceil}  \frac{\nu^*(D)^2}{\epsilon_k^2} \ln\frac{\nu^*(D) + \epsilon_k}{\epsilon_k} (d \ln\frac{\nu^*(D) + \epsilon_k}{\epsilon_k} + \ln\frac{k_0 - k + 1}{\delta}) )  \\
&\leq& O( \ln\frac{1}{\epsilon} \ln\frac{\epsilon+\nu^*(D)}{\epsilon} (d \ln\frac{\epsilon+\nu^*(D)}{\epsilon} + \ln\ln\frac{1}{\epsilon} + \ln\frac{1}{\delta}) + \frac{\nu^*(D)^2}{\epsilon^2} \ln\frac{\epsilon+\nu^*(D)}{\epsilon} (d \ln\frac{\epsilon+\nu^*(D)}{\epsilon} + \ln\frac{1}{\delta}) )\\
&\leq& O( \ln\frac{\epsilon+\nu^*(D)}{\epsilon} (\ln\frac{1}{\epsilon} + \frac{\nu^*(D)^2}{\epsilon^2}) (d\ln\frac{\epsilon+\nu^*(D)}{\epsilon} + \ln\frac{1}{\delta}) + \ln\frac{1}{\epsilon} \ln\frac{\epsilon+\nu^*(D)}{\epsilon} \ln\ln\frac{1}{\epsilon} )
\end{eqnarray*}

\subsection{Tsybakov Noise Conditon with $\kappa > 1$, Log-Concave Distribution}
We show in this subsection that under $(C_0,\kappa)$-Tsybakov Noise Condition, if $\calH$ is the class of homogeneous linear classifiers in $\R^d$, and $D_{\calX}$ is isotropic log-concave in $\R^d$, our label complexity bound is at most 
\[ O( \epsilon^{\frac{2}{\kappa}-2} \ln\frac{1}{\epsilon}(d\ln\frac{1}{\epsilon} + \ln\frac{1}{\delta}) ) \]
Recall by Lemma~\ref{lem:logconcavephi}, we have$\phi(C_0\epsilon_k^{\frac{1}{\kappa}},\frac{\epsilon_k}{256}) \leq C \epsilon_k^{\frac{1}{\kappa}} \ln\frac{1}{\epsilon_k}$ for some constant $C > 0$.
Applying Theorem~\ref{thm:labelcomplexitytnc}, the label complexity is:
\[ O( \sum_{k=1}^{\lceil \log\frac{1}{\epsilon} \rceil} (d\ln (\phi(C_0\epsilon_k^{\frac{1}{\kappa}},\frac{\epsilon_k}{256}) \epsilon_k^{\frac{1}{\kappa}-2}) + \ln(\frac{k_0 - k + 1}{\delta})) \phi(C_0\epsilon_k^{\frac{1}{\kappa}},\frac{\epsilon_k}{256}) \epsilon_k^{\frac{1}{\kappa}-2} )\]
This can be simplified to :
\begin{eqnarray*}
&& O( \sum_{k=1}^{\lceil \log\frac{1}{\epsilon} \rceil} (d\ln (\epsilon_k^{\frac{2}{\kappa}-2} \ln\frac{1}{\epsilon_k}) + \ln(\frac{k_0 - k + 1}{\delta})) \epsilon_k^{\frac{2}{\kappa}-2} \ln\frac{1}{\epsilon_k} )\\
&\leq& O ((\sum_{k=1}^{\lceil \log\frac{1}{\epsilon} \rceil} \epsilon_k^{\frac{2}{\kappa}-2}) \ln\frac{1}{\epsilon}(d\ln\frac{1}{\epsilon} + \ln\frac{1}{\delta})) \\
&\leq& O( \epsilon^{\frac{2}{\kappa}-2} \ln\frac{1}{\epsilon}(d\ln\frac{1}{\epsilon} + \ln\frac{1}{\delta}) )
\end{eqnarray*}